\documentclass{article}
\usepackage{kr}

\usepackage{times}
\usepackage{soul}
\usepackage{url}
\usepackage[hidelinks]{hyperref}
\usepackage[utf8]{inputenc}
\usepackage[small]{caption}
\usepackage{graphicx}
\usepackage{amsmath}
\usepackage{amsthm}
\usepackage{booktabs}
\usepackage{algorithm}
\usepackage{algorithmic}
\usepackage{latexsym}

\def\s5{{\mathbf{S5}}}

\def\calb{{\mathcal{B}}}
\def\calag{{\mathcal{A}}}

\def\calp{{\mathcal{P}}}

\def\calz{{\mathcal{Z}}}
\def\fm{{\,*_{\mbox{\scriptsize\rm fm}}\,}}
\def\ev{{\,*_{\mbox{\scriptsize\rm ev}}\,}}
\def\fb{{\,*_{\mbox{\scriptsize\rm rb}}\,}}

\def\B{{\mathbf{B}}}

\def\call{{\mathcal{L}}}

\def\Cn{{\mbox{\rm\em Cn\/}}}

\newcommand{\entails}{\ensuremath{\vdash}}

\newtheorem{lemma}{Lemma}

\newtheorem{theorem}{Theorem}

\usepackage{enumitem}

\newif\ifarxive
\arxivetrue

\ifarxive
\title{
A Study of Belief Revision  Postulates in Multi-Agent Systems (Extended Version)
}
\else
\title{
A Study of Belief Revision  Postulates in Multi-Agent Systems
}
\fi

\author{
    Michael Thielscher$^1$\and
    Tran Cao Son$^2$ \\
    \affiliations {
    $^1$University of New South Wales, Sydney, Australia\\
    $^2$New Mexico State University, Las Cruces, New Mexico, USA\\
    \emails{ 
    mit@unsw.edu.au, stran@nmsu.edu}
    }
}

\begin{document}

\maketitle

\begin{abstract}
We investigate the belief revision problem in epistemic planning, i.e., 
what will be the beliefs of all agents in a multi-agent system after an agent gains the belief in some state property.
Based on the standard representation
in epistemic planning of agents’ beliefs via a \emph{single multi-agent Kripke
model\/}, we generalize the classical AGM belief revision postulates to the multi-agent setting, with the aim to provide a formal framework for evaluating
dynamic epistemic reasoning frameworks in which the beliefs of all agents as the result of actions are computed. As an example of a simple operator that satisfies all of the generalized AGM postulates, we present generalized \emph{full-meet\/} multi-agent belief
revision. We moreover define a generalization of the standard postulates for \emph{iterated\/} revision, present a more sophisticated, event model based revision operator, and discuss the potential issues in defining an epistemic operator on Kripke
models that can satisfy all of the generalized postulates for iterated multi-agent belief revision.
\end{abstract}

\section{Introduction}

Multi-agent epistemic reasoning about actions and planning has garnered much research attention recently, as a formal framework for controlling heterogeneous, collaborative---or competitive---agents and robots. An example are service robots that interact with humans~\cite{BBMvD17}; other emerging applications for dynamic epistemic reasoning include the combination with large language models in order to plan interactions based on user beliefs~\cite{davila:logic}; model reconciliation~\cite{vaseil:dialec}; and reasoning about epistemic responsibility~\cite{chockl:respon} applied in legal contexts.

Dynamic epistemic logic (DEL) 
is considered a quasi standard as the most expressive formalism for modeling these domains~\cite{BolanderA11,buriga:episte}.
In this setting, the knowledge that different agents have, both of the environment and of each others' knowledge, is represented by a \emph{single Kripke model\/} of the possible states that the environment could be in according to the limited information of the agents. 
Such a multi-agent Kripke structure encodes all the beliefs of all agents, including those of higher order, i.e., beliefs about each other's beliefs. The latter is essential in most settings for dynamic epistemic reasoning when agents need to take into account what they know of the other (cooperating or competing) agents' knowledge or beliefs.
The use of a \emph{single\/} Kripke model to encode all beliefs of all agents in epistemic planning is motivated by efficiency\/: In order to predict the effects of a sequence of actions for the purpose of planning, it suffices to update one initial Kripke model action by action~\cite{BaralGPS22,buriga:episte}.

Actions are encoded by \emph{event models\/}, which similarly represent what the different agents know and can observe about the effects of an action. Combining one with the other, known as \emph{event model update\/}, results in a new Kripke model representing the updated knowledge after the action~\cite{boland:gentle}. 
As an example, the Kripke model on the left in the figure below represents the beliefs of two agents $a$ and~$b$ about the status of a coin ($h$: heads up); the event model in the middle (with colored arrows) encodes the pulic announcement of $\neg h$; and the rightmost model represents the beliefs of $a,b$ after the announcement\/:
\begin{center}
    \includegraphics[width=0.472\textwidth]{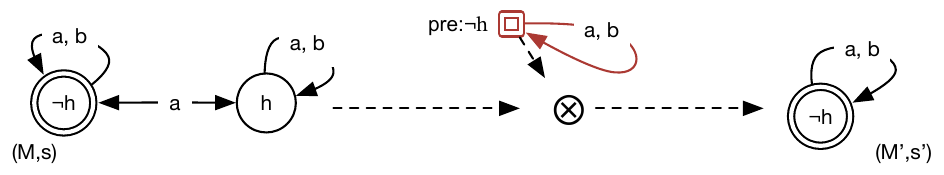}
\end{center}
Since event models can be complex to specify directly because they represent all possible views on an action occurrence, several high-level languages for describing actions in epistemic multi-agent environments have been developed~\cite{BaralGPS15,RajaratnamT21}. Multi-agent epistemic reasoning and planning is thus well-understood when agents have (incomplete) \emph{knowledge\/} of their environment; and event model update provides a clear semantics of how to update that knowledge after an action.

In reality, however, actions and sensors are often uncertain, so that agents can only hold \emph{beliefs\/}. These may be wrong, in which case they may need to be revised as a result of an action that reveals a false belief. But there is no established theory of how a Kripke structure should be updated to reflect this, and yet at the same time,
multi-agent belief revision is becoming ever more important in the area of agentic AI\/: 
in collaborative plan execution, for example, when agents have different beliefs about the precondition of actions they need to execute;  
when a chatbot tries to persuade a user to purchase a product and must be able to reason about the mental state of the user to decide on a `right' price; when
an AI system wishes to convince a human of its trustworthiness and has to reconcile its beliefs about the human's model with its own beliefs in order to provide explanations about its actions; and for
formal accounts of responsibility by reasoning about beliefs of plaintiffs and defendants.

Classical belief revision deals with the problem of identifying the beliefs of a single agent in the presence of a new piece of information.
In their seminal work, 
\citeauthor{AlchourronGM85} (\citeyear{AlchourronGM85})  
proposed the foundational properties, referred to as \emph{AGM Postulates}, that a rational belief revision operator should satisfy.
\citeauthor{DarwicheP97} (\citeyear{DarwicheP97}) added four properties for \emph{iterated\/} belief revision, commonly referred to as \emph{DP Postulates}, which focus on sequences of revisions.
The postulates have been extensively studied in the literature, including concrete revision operators 
based on the distance between models before and after a revision 
\cite{Dalal88,Satoh88}; belief revision based on possible worlds  \cite{KatsunoM92} and in the context of answer set programming   
\cite{AravanisP17,DelgrandeSTW08}; and
iterated belief revision and change for dynamic worlds \cite{BaltagS08,HunterD11,ShapiroPLL11,Peppas14,TardivoPSP21}.

Multi-agent revision has also been considered very early on; however, most of this work \cite{DragoniP94,DragoniG97,LiuW99,MalheiroJO94} focuses on maintaining the consistency of the local knowledge bases of agents in a distributed manner.
Proposals for changing the beliefs of multiple agents in dynamic environments mostly focus on belief \emph{updates\/} \cite{BaltagMS98,HerzigLM05,BenthemEK06,vBenthem07,vanDitmarschHK07,LoriniS21,BaralGPS22}. To the best of our knowledge, the only systematic study on the AGM and the DP postulates in belief revision in multi-agent environements is by
\citeauthor{aucher:genera}~(\citeyear{aucher:genera}), who considers belief sets represented by \emph{all\/} Kripke models consistent with a set of belief formulas. This allows for a straightforward application of the AGM postulates to a multi-agent setting but does not address the issue of revising a \emph{single\/} Kripke model as used in practical approaches to epistemic reasoning about actions and planning.

In this paper, \emph{we expand the logic of belief revision to the dynamic, multi-agent setting in order to provide a general, systematic framework for assessing multi-agent belief revision operators that are defined over a single Kripke model as used in dynamic epistemic reasoning about actions and planning}. To this end, we develop a suitable multi-agent generalization of the AGM postulates.
We present a simple, generalized so-called \emph{full-meet\/} multi-agent belief revision operator and prove that it satisfies the generalized AGM postulates. We also generalize the DP postulates to iterated multi-agent belief revision and show that full-meet satisfies all but one of them. We then develop a more sophisticated, event model-based belief revision operator and analyze different strategies to define such operators and discuss their consequences in satisfying the DP posulates.

\ifarxive
This is an extended version, with full proofs in the appendix, of a paper accepted at KR 2026.
\else
All omitted proofs are available in an accompanying technical report~\cite{ThielscherSon2026:arxiv}.
\fi

\section{Background}

We begin with a concise outline of the formal foundations of dynamic epistemic reasoning and belief revision

\subsubsection{Dynamic Epistemic Reasoning}

A {\em multi-agent} domain is defined by a pair $\langle \calag, \calp \rangle$ where $\calag$ is a finite and non-empty set of agents    
and $\calp$ a set of propositions.
\emph{Belief formulas} over $\langle \calag, \calp \rangle$ are defined by the  BNF: \quad  
 $
   \varphi \:\: {:}{:}{=} \:\:    p \mid \neg \varphi \mid (\varphi \wedge \varphi) \mid (\varphi \vee \varphi)  \mid \B_i\varphi
$\quad  
where $p \in \calp$ and $i \in \calag$.
Standard connectives like $\rightarrow$ and $\leftrightarrow$ are used through their usual abbreviations. A belief formula which does not contain any occurrence of $\B_i$ is referred to  
as a proposition formula. $\call_\calp$  
(resp. $\call_\calag$) denotes the set of proposition formulas over $\calp$ (resp. the set of belief formulas over $\langle \calag, \calp \rangle$).

Satisfaction of belief formulas is defined over \emph{pointed Kripke structures} (a.k.a.\ pointed Kripke models or epistemic states) \cite{FaginHMV95}. A Kripke structure/model $M$ is a tuple $\langle W, \{R_{a}\}_{a \in \calag},\pi\rangle$,  where   $W$ is a set of worlds, $\pi: W \mapsto 2^\calp$ is a function that associates an interpretation of $\calp$ to each  element of $W$, 
and  for $a \in \calag$, $R_{a} \subseteq W \times W$ is a binary relation over $W$. 
We write $R_a(u,w)$ and use this notation interchangeably with $(u,w) \in R_a$.  
For $u \in W$ and $\varphi \in \call_\calp$,  $M[\pi](u)$ and $M[\pi](u)(\varphi)$ denote the interpretation associated with $u$ via $\pi$ and the truth value of $\varphi$ with respect to $M[\pi](u)$. For a world $s \in W$, referred to as \emph{true state of the world}, $(M,s)$ is a \emph{pointed Kripke structure}.  

The satisfaction relation $\models$ between a state $(M,s)$ and belief formulas is defined as follows: (\emph{i}) $(M,s) \models p$ if \mbox{$p \in \calp$} and $M[\pi](s) \models p$;
(\emph{ii}) $(M,s) \models \B_{i}\varphi$ if  $\forall t.[(s,t) \in R_{i} \Rightarrow (M,t) \models \varphi]$;
(\emph{iii}) $(M,s) \models \neg\varphi$ if $(M,s) \not\models \varphi$;
(\emph{iv})~$(M,s) \models \varphi_1 \vee \varphi_2$ if $(M,s) \models \varphi_1$ or $(M,s) \models \varphi_2$;
(\emph{v})~$(M,s) \models \varphi_1 \wedge \varphi_2$ if $(M,s)\models \varphi_1$ and $(M,s) \models \varphi_2$. 

Two pointed Kripke structures $\langle W, \{R_{a}\}_{a \in \calag},\pi\rangle$ and $\langle W', \{R'_{a}\}_{a \in \calag},\pi'\rangle$ are \emph{bisimilar\/} if there is a relation~$\calz\subseteq W\times W'$ such that for all $(w,w')\in\calz$\/:
(\emph{i})  $M[\pi](w)=M'[\pi'](w')$; 
(\emph{ii}) for each $w_1\in W$ such that $R_a(w,w_1)$, $R'_a(w',w_1')$ for some $(w_1,w_1')\in\calz$;
(\emph{iii}) for each $w'_1\in W'$ such that $R'_a(w',w'_1)$, $R_a(w,w_1)$ for some $(w_1,w_1')\in\calz$. 

For a set of formulas $A$, $Cn(A)$ denotes the set of logical consequences of $A$.
$Cn$ is assumed to be supraclassical, i.e., if $p$ can be derived from $A$ 
 by classical truth-functional logic, then $p \in Cn(A)$. 
$Cn$ satisfies the following properties: (\emph{i})~$A \subseteq Cn(A)$; (\emph{ii})~if $A \subseteq B$ then $Cn(A) \subseteq Cn(B)$; (\emph{ii}) $Cn(A) = Cn(Cn(A))$.
We say that $A$ is a belief set if and only if $A = Cn(A)$. 

For a set~$K$ of formulas, $K \vdash p$ (resp. $K  \not\vdash p$)
 is an alternative notation for $p {\in} Cn(K)$ (resp. $p \not\in Cn(K)$).
 $Cn(\emptyset)$ is the set of tautologies.
The \emph{expansion\/} of $K$ by a formula~$p$, i.e., the operation that just adds $p$ 
 and removes nothing, is denoted $K+p$ and defined by: $K + p = Cn(K \cup \{p\}$).

\subsubsection{Logic of Belief Revision and AGM Postulates} 
We follow \citeauthor{DarwicheP97}~(\citeyear{DarwicheP97}) and state the standard postulates on the basis of \emph{epistemic states\/}~$M$, which implicitly determine a set of beliefs $K_M$. Two epistemic states are \emph{equivalent\/}, written $M_1\equiv M_2$, iff $K_{M_1}=K_{M_2}$, that is, they entail the same beliefs.
The intuitive meaning of $M * p$ is to revise the beliefs so as to ensure that $K_{M *p}$ contains~$p$ and is consistent (unless $p$ is inconsistent).
\citeauthor{AlchourronGM85}~(\citeyear{AlchourronGM85}), a.k.a.\ AGM, proposed the following eight basic postulates for one-shot belief revision\/:
\begin{itemize} 
    \item \emph{Closure}: $K_{M * p} = \Cn(K_{M*p})$
    \item \emph{Success}: $p \in K_{M*p}$
    \item \emph{Inclusion}: $K_{M * p} \subseteq K_M + p$
    \item \emph{Vacuity}: if $\neg p \not\in K$ then $K_{M * p} = K_M + p$
    \item \emph{Consistency}: $K_{M * p}$ is consistent if $p$ is consistent    
    \item \emph{Extensionality}: if $\entails p \leftrightarrow q$ then $M * p \equiv M * q$
    \item \emph{Superexpansion}: $K_{M * (p \wedge q)} \subseteq K_{M * p} + q$
    \item \emph{Subexpansion}: $K_{M * p} + q \subseteq K_{M * (p \wedge q)}$ if $\neg q \not\in K_{M* p}$
\end{itemize}
\citeauthor{DarwicheP97}~(\citeyear{DarwicheP97}) augmented the AGM framework by four more postulates for \emph{iterated\/} revision\/:
\begin{description}
\item[(DP1)] if $q \entails p$ then $K_{(M*p) * q} = K_{M*q}$
\item[(DP2)] if $q \entails \neg p$ then $K_{(M*p) * q} = K_{M*q}$
\item[(DP3)] if $p\in K_{M*q}$ then $p\in K_{(M*p) * q}$
\item[(DP4)] if $\neg p\not\in K_{M*q}$ then $\neg p\not\in K_{(M*p) * q}$
\end{description}
Another standard postulate, which strengthens (\textbf{DP3}) and (\textbf{DP4}), is the following~\cite{jin2007iterated}\/: 

\begin{description}
\item[(IN)]
\emph{Independence}: if $ \neg p\not\in K_{M*q}$ then $p\in K_{(M*p) * q}$
\end{description}

\section{Multi-agent Belief Revision (MBR) --- Basic Concepts} \label{sec:MBR}

We begin by discussing in detail the foundations for multi-agent belief revision in the context of dynamic epistemic reasoning, where single Kripke models represent the beliefs of all agents. In the section that follows, we then develop generalizations of the AGM postulates to this setting. 

\subsubsection{Belief sets} As customary in epistemic reasoning about actions and planning, a \emph{belief set\/} in MBR shall be represented by a single, pointed Kripke structure $(M,s)$ and is the set of all formulas entailed by $(M,s)$\/:
\[
K_{(M,s)}\,=\,\{\varphi\in \call_\calag\mid (M,s)\,\models\,\varphi\}
\]
From now on, we will assume that for any pointed Kripke structure $(M,s)$ in discussion,
$s$ is the true state of the world.  

\begin{figure}[t]
    \centering
\raisebox{-0.5\height}{%
    \includegraphics[width=0.46\linewidth]{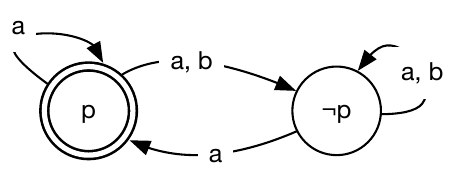}%
  }
  \ \ \ \ \
  \raisebox{-0.5\height}{%
    \includegraphics[width=0.46\linewidth]{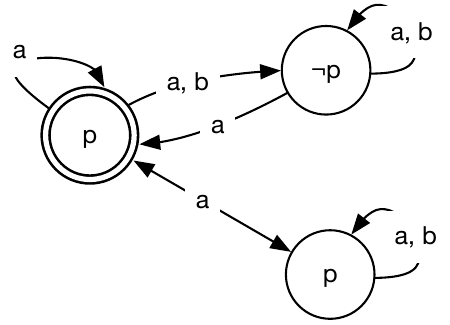}%
  }
        \caption{Two pointed Kripke structures $(M_1,s)$ and $(M_2,s)$ with the same ``true'' world~$s$ (marked with a double circle  --- so both structures agree that $p$ is actually true) but which determine different multi-agent belief sets $K_1=K_{(M_1,s)}$ and $K_2=K_{(M_2,s)}$.}
    \label{fig:belsets}
\end{figure}

As an example, Fig.~\ref{fig:belsets} shows two pointed Kripke structures, with agents~$a$ and~$b$, side by side. The egdes labelled $a,b$ from the actual world into a world in which $\neg p$ holds indicate that both agents consider it possible that $p$ is false. In fact, agent $b$ (falsely) believes in $\neg p$ because this is the only $b$-accessible world from the true world, $s$. Agent $a$, on the other hand, considers both~$p$ and~$\neg p$ possible because there is also a self-loop on~$s$ labelled $a$. Also, in both belief sets $K_1$ and $K_2$, $b$ believes that $a$ does not believe in either, that is, $\B_b[(\neg\B_a p)\wedge(\neg\B_a\neg p)]$, because from the only $b$-accessible world there is one $a$-accessible world with $p$ and one with $\neg p$. The two belief sets differ in another higher-degree belief, however\/: Only in~$K_1$, $a$~believes that $b$ believes $\neg p$. Formally, $\B_a \B_b\, \neg p\in K_1$ whereas $\neg\B_a \B_b\, \neg p\in K_2$, because in $(M_2,s)$ there is an additional $a$-accessible world in which $b$ does believe $p$.

Note that \emph{bisimilar\/} multi-agent Kripke models define the same belief set provided the two designated worlds coincide in their interpretation. When two structures are \emph{dissimilar\/}, they still induce the same belief set if they entail the same belief formulas from the respective true state of the world.

\subsubsection{Deductive closure of a belief set} For a belief set~$K$, the deductive closure $\Cn(K)$ is given by the entailment relation~``$\entails$''. Any belief set therefore satisfies $\Cn(K)=K$.

\subsubsection{Consistent beliefs/belief sets} Agents can have inconsistent beliefs\/: By definition, $\B_a\bot\in K_{(M,s)}$ if, and only if, $R_a(s)=\emptyset$ in~$M$. Agents can also believe that other agents have inconsistent beliefs etc., but a belief set itself is always consistent because $M[\pi](s)\not\models\bot$ for any pointed Kripke structure. Hence, the concept of consistency in the classical AGM postulates needs to be adapted to consistency of the beliefs \emph{of individual agents\/} in the multi-agent setting when representing a belief set by a single Kripke structure.

\subsubsection{Revision} With the aim to model actions that change the beliefs of agents 
(e.g., sensing actions, announcement actions)  
in a multi-agent environment, we consider revising belief sets to reflect the result of one agent making an observation or receiving some information 
about the environment. Formally, a belief set will be revised by \emph{first-degree belief formulas\/}. These are defined as 
\[
   \calb_{\calag,\calp} = \{\B_a\varphi \mid a\in\calag,\ \varphi\ \mbox{proposition formula over}\ \calp\}
\]
The belief set after revision by a first-degree belief $\B_a\varphi\in\calb_{\calag,\calp}$ should again be represented by a single, pointed Kripke structure, denoted by $(M,s)*\B_a\varphi$. We will simply write $K*\B_a\varphi$ to refer to the belief set $K_{(M,s)*\B_a\varphi}$ when it is clear from the context that $(M,s)$ is the underlying Kripke model of the belief set $K$.

\subsubsection{Subset relation over belief sets} The classical postulates also require us to define the concept of a subset relation among belief sets. Because every belief set is represented by a single Kripke structure, we cannot define this relation based on the set of all formulas entailed, since otherwise the subset relation would be satisfied only if the two Kripke structures entail identical sets of formulas. Therefore, and in line with the definiton of revision formulas, we consider all formulas of the form $\B_a\varphi$, where $a\in\calag$ and $\varphi$ is a proposition formula over~$\calp$, when comparing two belief sets\/:
\[ \begin{array}{c}
   K_{(M_1,s_1)}\subseteq K_{(M_2,s_2)} \\ \mbox{iff} \\
   \forall\, \B_a\varphi\in \calb_{\calag,\calp}.
   [(M_1,s_1)\models \B_a\varphi\Rightarrow (M_2,s_2)\models \B_a\varphi]
\end{array} \]
Recall, for example, the belief sets represented in Fig.~\ref{fig:belsets}. Although $K_1\not=K_2$, both of them entail the same first-degree belief formulas. This follows from the fact that for both Kripke structures $(M,s)$, we have
\[ \begin{array}{l}
   \exists (s,w),(s,w')\in R_a.\ M[\pi](w)\models p\,\wedge\, M[\pi](w')\models \neg p \\
   \wedge\ \forall (s,w)\in R_b.\ M[\pi](w)\models p
\end{array} \]
Hence, in both models it holds that $(M,s)\models \B_a\varphi$ iff $\varphi$ is a tautological propositional formula while
$(M_i,s)\models \B_b\varphi$ iff $p\entails\varphi$. Consequently,
$K_1\subseteq K_2$ and $K_2\subseteq K_1$.

\subsubsection{Minimal belief sets} A consequence of the above definition is that the ``smallest'' (w.r.t.\ the subset relation) representable belief sets are exactly those that include only first-degree belief formulas of the form $\B_a\varphi$ with $\varphi$ a propositional tautology over $\calp$. There are different Kripke structures that can be used to represent this set; a \emph{generic\/} minimal Kripke structure can be constructed as follows\/:
\mbox{$M_\emptyset=\langle W,\{R_a\}_{a\in\calag},\pi\rangle$} with $W=2^{\calp}$, $R_i=W\times W$ for all $i\in\calag$, and $\pi(w)=w$.
\begin{lemma}
For any $s\in 2^{\calp}$ we have that $K_{(M_\emptyset,s)}\subseteq K$ for all belief sets~$K$.
\end{lemma}
It is worth stressing that not all smallest belief sets are equal as they can be based on structurally different Kripke models and hence contain different nested beliefs. For example, $(M_\emptyset,s)$ always entails $\B_a\neg\B_b\varphi$ for all agents \mbox{$a,b\in\calag$} and non-tautological proposition formulas~$\varphi$, that is, every agent believes that no agent believes in anything other than tautological properties about the environment. This may not be the case in other minimal belief sets.

\subsubsection{Expansion} A key concept in the classical AGM postulates is the expansion of a belief set by a new belief. The intuition behind this concept is to add a new belief while retaining the existing beliefs, together with all the logical consequences of the old and new beliefs, but in a minimal fashion. The intuition behind the following generalization of this concept to multi-agent beliefs given by a pointed Kripke structure, is that $K+\B_a\varphi$ is obtained by constructing a new structure that is a combination of\/: (1)~$M$ (the old structure); (2)~a new true state of the world $s'$; and (3)~a ``replica'' structure obtained from~$M$ by (i)~removing all links labeled $a$ into a world in which $\varphi$ is false, and (ii)~adding links labeled $a$ going from $s'$ to worlds in which there is a link labeled $a$.

Formally, let $M = (W, \{R_a\}_{a \in \mathcal{A}}, \pi)$, then \emph{expanding\/} a pointed Kripke model $(M,s)$ by a first-degree belief formula results in the pointed Kripke model $(M,s) + \B_a \varphi = (M',s')$ with $M' = (W', \{R'_a\}_{a \in \mathcal{A}}, \pi')$ such that
    \begin{itemize}
        \item  $W' = W \cup W^r \cup \{s'\}$ with $W^r  = \{s^r \mid s \in W\}$ --- the \emph{replica\/} of $W$ --- and where 
        $s'$ is a new world symbol that does not occur in $W \cup W^r$;

        \item for $w \in W$ and $w' \in \{w, w^r\}$, $\pi'(w') = \pi(w)$; and $\pi'(s') = \pi(s)$;

        \item for $x \in \mathcal{A} \setminus \{a\}$, 
        
        \begin{itemize}
            \item if $(u, v) \in R_x$ then  $(u,v)$ and $(u^r,v^r)$ belong to $R'_x$,

            \item if $(s, v) \in R_x$ then $(s', v)$ belongs to $R'_x$;
        \end{itemize} 
        \item for $x = a$,
        \begin{itemize}
            \item if $(u, v) \in R_x$ and $\pi(v)\models \varphi$ then  $(u,v)$ and $(u^r,v^r)$ belong to $R'_x$,
            \item if $(u, v) \in R_x$ and $\pi(v)\not\models \varphi$ then  $(u,v)$ is in $R'_x$,
            \item if $(s, v) \in R_x$ and $\pi(v)\models \varphi$ then $(s',v^r)$ is in  $R'_x$.
        \end{itemize}
    \end{itemize} Let $K$ be a belief set represented by the pointed Kripke model $(M,s)$, and let $\B_a\varphi\in\calb_{\calag,\calp}$, then the \emph{expansion\/} $K+\B_a\varphi$ is defined as the belief set $K_{(M',s')}$ where $(M,s')=(M,s)+\B_a\varphi$.

\begin{figure}[t]
    \centering
    \includegraphics[width=0.73\linewidth]{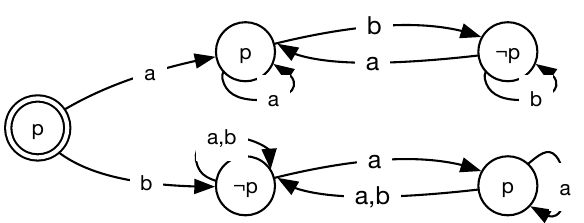}
    \caption{The Kripke model resulting from expanding $(M_1,s)$ in Fig.~\protect\ref{fig:belsets} by $\B_a p$. The new designated world~$s'$  (marked with a double circle) is linked via~$a$ to the replica of the old Kripke structure (top) and via~$b$ to the old Kripke structure itself (bottom). All $a$-links to $\neg p$-worlds have been removed in the replicated part. As a result, $a$ now believes in~$p$ (but also still in $\B_b\neg p$).}
    \label{fig:revisedbelsets}
\end{figure}

As an example, Fig.~\ref{fig:revisedbelsets} depicts the results of expanding the Kripke structure to the left in Fig.~\ref{fig:belsets} by $\B_a p$. Obviously, agent~$a$ now believes in $p$, i.e.\ $\B_a p\in K_1+\B_a p$, since the only $a$-accessible world from $s'$ satisifies~$p$. We can also see that $b$ retains exactly her old beliefs, e.g., $\B_b\neg p\in K_1+\B_A p$ and also $\B_b((\neg\B_a p)\wedge(\neg\B_a\neg p))\in K_1+\B_a p$. This is so because the $b$-accessible worlds from $s'$ are exactly those that were previously accessible from~$s$, and with the same structure. Meanwhile, $a$ also retained his belief that $b$ believes in~$\neg p$, that is, $\B_a\B_b\neg p\in K_1+\B_a p$.

The next lemma shows that expanding a belief set with $\B_a \varphi$ does not change the first-degree beliefs of other agents. 
\begin{lemma}
\label{lem:belief-other}
For any formula $\psi$ and agent $x \in \mathcal{A} \setminus \{a\}$, 
$(M,s) \models B_x \psi$ iff $(M',s') \models B_x \psi$. 
\end{lemma}
\begin{proof} 
By construction of $(M,s) + \B_a \varphi$, 
we have that $(s',u) \in R'_x$ iff $(s,u) \in R_x$. Therefore, 
$(M,s) \models B_x \psi$ iff $(M',s') \models B_x \psi$. 
\end{proof} 
The definition of expansion applies to any multi-agent Kripke structure and new first-degree belief, including when no world in the model satisfies the new belief. The example depicted in Fig.~\ref{fig:inconsbeliefs} illustrates that in this case the resulting expanded model still does not include a world that satisfies the new belief, here\/: $\B_a \neg p$. This implies that there are no $a$-reachable worlds at all from the new designated state, which in turn means that agent~$a$ ends up with inconsistent beliefs (and hence, in particular, believes~$\neg p$). This is very much in the spirit of expansion in the classical case, when expanding by a new belief that is inconsistent with the current ones results in an inconsistent belief set~\cite{AlchourronGM85}. In the multi-agent generalization, all other agents maintain their beliefs, however (cf.~Lemma~\ref{lem:belief-other}).

\begin{figure}
    \centering
    \includegraphics[width=0.95\linewidth]{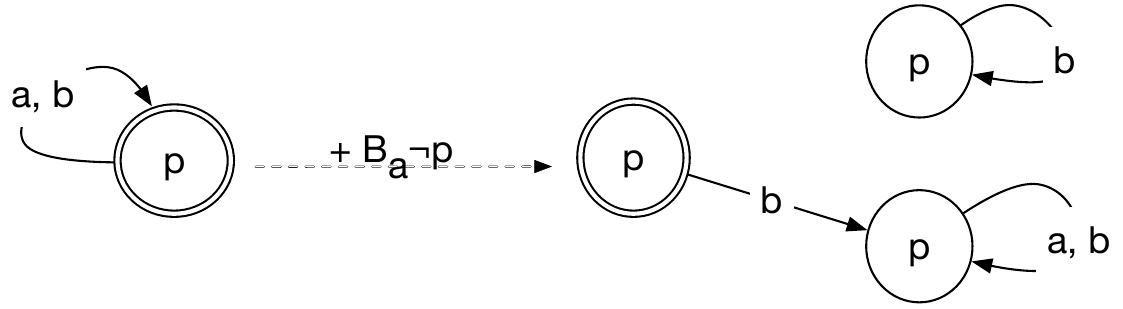}
    \caption{A Kripke structure in which $p$ is common belief among agents~$a,b$ is expanded by the new belief~$\B_a\neg p$. Agent~$b$'s first- and higher-order beliefs remain unchanged thanks to the retained, $b$-reachable copy of the old structure (bottom) whereas there is no $a$-reachable world in the replica (top), hence $\B_a\bot$ after expansion.}
    \label{fig:inconsbeliefs}
\end{figure}
It is easy to prove that expanding a belief set with $\B_a \varphi$ always results in agent~$a$ believing~$\varphi$ when $a$~ends up with consistent beliefs; and if the result is that $a$ has inconsistent beliefs, then again, vacuously, $a$ believes in $\varphi$.  
\begin{lemma}
\label{lem:belief-agent}
$(M,s) + \B_a\varphi \models \B_a \varphi$.
\end{lemma}
\begin{proof} 
  By definion of the entailment of formulas from a pointed Kripke structure, 
    if $R'_a(s') = \emptyset$ then $(M',s')\models\bot$ so that $(M',s') \models \B_a \psi$ for any formula $\psi$, hence specifically, $(M',s') \models \B_a \varphi$. 

    Assume that $R'_a(s') \ne \emptyset$. By construction, 
    for every $u^r$ such that $(s',u^r) \in R'_a$, we have that $(s,u) \in R_a$ and 
    $\pi(u) \models \varphi$. This implies that $\pi(u^r) \models \varphi$, and thus, 
    $(M',s') \models \B_a \varphi$ because  $R'_a(s') \ne \emptyset$.
\end{proof}
The next lemma shows that the first-degree beliefs of agent~$a$ are exactly the logical consequences of its old beliefs plus the new belief.
\begin{lemma}
\label{lem:first-degree-belief-agent}

If $R'_a(s') \ne \emptyset$ then $(M',s') \models \B_a \psi$ if, and only if, $\psi \in
\Cn\left(\{\varphi\}\cup \{\phi\in\calb_{\calag,\calp}\mid \B_a\phi\in K_{(M,s)}\}\right)$. 
\end{lemma}
\begin{proof} 
The proof of this lemma is similar to the proof of Lemma~\ref{lem:belief-agent} with the observation that 
for $\psi \in \Cn\left(\varphi\cup \{\phi\in\calb_{\calag,\calp}\mid \B_a\phi\in K_{(M,s)}\}\right)$, 
$(s',u^r) \in R'_a$ and $\pi'(u^r) \models \psi$ iff $(s,u) \in R_a$ and 
$\pi(u) \models \psi$.
\end{proof}

\section{Generalizing AGM Postulates}

Using the basic formal concepts for multi-agent belief sets represented by a single, pointed Kripke strucure as developed in the previous section, we can now define generalizations of the standard postulates, beginning with the AGM potulates for one-shot revision.

\subsubsection{Closure} A revised belief set is deductively closed, denoted by $\Cn$, under the modal logic being interpreted\/:
\[ K * \B_a\varphi = \Cn(K * \B_a\varphi) \]
Since the set of belief formulas entailed by a single Kripke structure is deductively closed, this postulate holds when a revised belief set is represented by a pointed Kripke model.

\subsubsection{Success}

The result of revising a Kripke model by a first-degree belief formula should include the new belief\/:
\[ \B_a\varphi \in K*\B_a\varphi \]

\subsubsection{Inclusion} A revised Kripke model should only contain first-degree belief formulas that would be included in the expanded Kripke model\/:
\[ K * \B_a\varphi \subseteq K + \B_a\varphi \]
From Lemma~\ref{lem:belief-other} and~\ref{lem:first-degree-belief-agent}
it follows that under the Inclusion principle, a revised Kripke model contains only first-degree belief formulas that follow logically from the old and new first-degree beliefs.

Worthy of note, this generalized Inclusion postulate does not stipulate any requirements about second- or higher-degree belief formulas. In particular it allows for another agent,~$b$, to change her beliefs about agent~$a$ believing in~$\varphi$. Hence, $K*\B_a\varphi$ may contain beliefs that are not included in $K+\B_a\varphi$ if these are not first-order beliefs.

\subsubsection{Vacuity} If a Kripke model is revised by a first-degree belief that is consistent with the current beliefs, then the result should contain all first-degree belief formulas that are included in the expanded Kripke model\/:
\[ \B_a\neg \varphi \not\in K\ \Rightarrow\ K + \B_a\varphi \subseteq K * \B_a\varphi \]
By Lemma~\ref{lem:belief-other} and~\ref{lem:first-degree-belief-agent}
it follows that under the Vacuity principle, a revised Kripke model contains all first-degree belief formulas that follow logically from the old and new first-degree beliefs, provided the old beliefs did not include the opposite of the new belief.

Similar to the generalized Inclusion principle, the Vacuity postulate does allow for belief revision operators in which other agents change their belief about the belief of agent~$a$ in~$\varphi$ as a result of this revision.

It should also be noted that Vacuity and Inclusion together are a weaker requirement than stipulating that revision be identical to expansion in case a new belief is consistent with the old ones. They merely postulate that the first-degree beliefs are the same, while they do not demand anything about other belief formulas. In particular, they do not prescribe the specific structure from our definition of expansion of a pointed Kripke model with a new belief (cf.\ Section~\ref{sec:MBR}).

\subsubsection{Consistency 1} Any revision by a logically consistent belief should result in a consistent belief for the agent\/:
\[
   \not\entails \varphi\rightarrow\bot\ \Rightarrow\
   \B_a\bot \not\in K * \B_a\varphi
\]
It is worth noting that this does not postulate overall consistency of beliefs as it cannot be generally assumed that one agent changing their beliefs would always mean that any other agent that may have had inconsistent beliefs would automatically end up with consistent beliefs too. However, while it is possible that other agents change their beliefs about agent~$a$'s beliefs, it is reasonable to postulate that they do not end up with inconsistent beliefs as a result if they had consistent beliefs beforehand. For this reason, we suggest the following additional postulate on consistency.

\subsubsection{Consistency 2} If a new belief $\B_a\varphi$ is consistent, then any agent with consistent beliefs will have consistent beliefs after the revision\/: 
\[
   \B_b\bot\not\in K\ \wedge\ \not\entails \varphi\rightarrow\bot\ \Rightarrow\
   \B_b\bot \not\in K * \B_a\varphi
\]

\subsubsection{Extensionality} If two new belief formulas are logically equivalent w.r.t.\ the underlying modal logic, then a belief set revised by either of the two should give the same result\/:
\[ \entails \B_a\varphi \leftrightarrow \B_b\psi\ \Rightarrow\ K * \B_a\varphi = K * \B_b\psi \]
It is easy to see that $\B_a\varphi$ and $\B_b\psi$ are logically equivalent under \emph{any\/} pointed Kripke structure if, and only if, $a=b$ and $\varphi$ and $\psi$ are logically equivalent proposition formulas.

\subsubsection{Superexpansion} Revising a belief set by a conjunction~$\varphi\wedge\psi$ of two new beliefs of an agent~$a$ should not result in more first-degree beliefs than the expansion by $\B_a\psi$ of the result of revising by $\B_a\varphi$\/:
\[ K * \B_a(\varphi \wedge \psi) \subseteq (K * \B_a\varphi) + \B_a\psi \]
Similarly to Inclusion and Vacuity, this generalized postulate does not make any assumptions about second- or higher-degree belief formulas of any agent. This is also true for the following counterpart, subexpansion.

\subsubsection{Subexpansion}

If the second belief~$\B_a\psi$ is consistent with the result of revising a belief set by the first belief $\B_a\varphi$, then
revision by the conjunction of the two should not result in fewer first-degree beliefs than the expansion by $\B_a\psi$ of the result of revising by $\B_a\varphi$\/:
\[ \B_a\neg \psi \not\in K* \B_a\varphi \Rightarrow
(K * \B_a\varphi) + \B_a\psi \subseteq K * \B_a(\varphi \wedge \psi)
\]

\section{Multi-agent Belief Revision Operators}

Having defined generalized AGM postulates for MBR, in this section we present a generalization of the well-known ``full-meet'' revision operator for classical Belief Revision~\cite{alchourron1982logic} and show that it satisfies all generalized AGM postulates. 

We define this \emph{multi-agent full meet revision}, denoted by the opeartor name~$\fm$, as follows\/:
\[
   K_{(M,s)}\fm\B_a\varphi\,=\! \left\{\!\!\!
   \renewcommand{\arraystretch}{0.7}
   \begin{array}{ll}
     K_{(M,s)}+\B_a\varphi & \!\!\!\mbox{if } \B_a\neg\varphi\not\in K_{(M,s)}\!\!\!\! \\ \\
     K_{(M_\emptyset,s)}+\B_a\varphi & \!\!\!\mbox{otherwise}
   \end{array}
   \right.
\]
The principle behind this definition is the same as for classical full-meet revision~\cite{AlchourronGM85}. If a new belief is consistent with the current beliefs, the underlying Kripke structure is simply expanded by that belief. Otherwise, the new belief is incorporated in the most conservative manner by starting with the minimal belief set~$K_{(M_\emptyset,s)}$ (cf.\ Section~\ref{sec:MBR}), in which it is common knowledge that no agent believes in anything but tautological proposition formulas, and then expanding the undelying Kripke model~$(M_\emptyset,s)$ by the new belief.

It should be noted, however, that unlike with the classical full-meet belief revision operator, the belief set resulting from revision by a belief that is inconsistent with the old beliefs does entail more beliefs than follow logically from the new one. This is so because the generic Kripke structure~$(M_\emptyset,s)$ makes strong assumptions about higher-degree beliefs. 
For example, full-meet revision of the Kripke structure depicted on the left-hand side in Figure~\ref{fig:inconsbeliefs} by $\B_a\neg p$ results in agent~$b$ not only \emph{losing\/} both her first-degree belief in~$p$ as well as her higher-degree belief that $p$ is common knowledge, but also \emph{gaining\/} second-order beliefs of ``ignorance'', such as, say, $\B_b\left(\neg\B_a p\wedge\neg\B_a\neg p\right)$.

This notwithstanding, the generalized full-meet operator provably satisfies all of the generalized AGM postulates.

\begin{theorem}
$\fm$ satisfies the generalized AGM postulates.
\end{theorem}
\noindent 
\emph{Proof.} Due to limited space, we 
include below only the proof for 
\textbf{Superexpansion}
as it is somewhat more complicated than the others. Detailed proofs for all theorems and lemmas
\ifarxive
can be found in the appendix.
\else
 in the paper are available in the supplementary report~\cite{ThielscherSon2026:arxiv}.
\fi
\begin{enumerate}
   \item Suppose $\B_a\neg(\varphi\wedge\psi)\not\in K$. Since $K$ is deductively closed, it follows that $\B_a\neg\varphi\not\in K$.
   Hence, agent~$a$ has consistent beliefs in $K+\B_a(\varphi\wedge\psi)$, in $K+\B_a\varphi$, and in $(K+\B_a\varphi)+\B_a\psi$.
   By Lemma~\ref{lem:first-degree-belief-agent} it follows that $(M,s)+\B_a(\varphi\wedge\psi)\models\B_a\chi$ iff $\chi\in\Cn(\{\varphi\wedge\psi\}\cup\{\phi\in\calb_{\calag,\calp} \mid\ \B_a\phi\in K\})$. This is equivalent to
   $\chi\in\Cn(\{\psi\}\cup\Cn(\{\varphi\}\cup\{\phi\in\calb_{\calag,\calp} \mid\ \B_a\phi\in K\}))$, which in turn by Lemma~\ref{lem:first-degree-belief-agent} is equivalent to
   $\left(K+\B_a\varphi\right) + \B_a\psi\models\B_a\chi$.
      By Lemma~\ref{lem:belief-other} it follows that all other agents too have the same beliefs in $K+\B_a(\varphi\wedge\psi)$ and in $(K+\B_a\varphi)+\B_a\psi$.
      Hence, $K \fm \B_a(\varphi \wedge \psi) \subseteq (K \fm \B_a\varphi) + \B_a\psi$.

   \item Suppose $\B_a\neg(\varphi\wedge\psi)\in K$. By definition of $(M_\emptyset,s)+\B_a(\varphi\wedge\psi)$ and Lemma~\ref{lem:first-degree-belief-agent} it follows that, for any first-degree belief, $\B_a\phi\in K\fm \B_a(\varphi\wedge\psi)$ iff $\phi\in\Cn(\varphi\wedge\psi)$, and for $x\in\calag\setminus\{a\}$, $\B_x\phi\in K\fm \B_a(\varphi\wedge\psi)$ iff $\models\phi$.

      We distinguish two cases\/:
      If $\B_a\neg\varphi\not\in K$, then $(K\fm\B_a\varphi)+\B_a\psi=(K+\B_a\varphi)+\B_a\psi$. From $\B_a\neg(\varphi\wedge\psi)\in K$ it follows that $\B_a\neg\psi\in K+\B_a\varphi$. Hence, $a$ has inconsistent beliefs in $(K+\B_a\varphi)+\B_a\psi$ while the beliefs of all other agents $x\in\calag\setminus\{a\}$ are the same in $(K+\B_a\varphi)+\B_a\psi$ and $K+\B_a(\varphi\wedge\psi)$. 
      If, on the other hand, $\B_a\neg\varphi\in K$, then $(K\fm\B_a\varphi)+\B_a\psi=(K_{(M_\emptyset,s)}+\B_a\varphi)+\B_a\psi$.
      Hence, $a$'s first-degree beliefs in $(K\fm\B_a\varphi)+\B_a\psi$ are exactly the logical consequences of $\varphi\wedge\psi$ while the beliefs of all other agents $x\in\calag\setminus\{a\}$ are the same in $(K\fm \B_a\varphi)+\B_a\psi$ and $K+\B_a(\varphi\wedge\psi)$.
   Thus, $K \fm \B_a(\varphi \wedge \psi) \subseteq (K \fm \B_a\varphi) + \B_a\psi$. \hfill{$\Box$}
\end{enumerate}

While satisfying all our generalized postulates for one-shot revision, multi-agent full-meet revision makes for a very drastic revision in cases where simple expansion would leave an agent with inconsistent beliefs. We introduce a more refined alternative in Section~\ref{sec:assessment}; prior to this, we first turn to the question of successive multi-agent revisions. 

\section{Generalized Postulates for Iterated MBR} \label{sec:dp}

We complete our framework for belief revision for multi\-agent epistemic reasoning and planning by generalizing the standard postulates for \emph{iterated\/} revision.

\subsubsection{DP1 -- Successive revision respect}

Revision by one belief followed by revising by a stronger belief makes the first revision redundant\/: 
\[
   \entails \B_b\psi \rightarrow \B_a\varphi\ \Rightarrow\ (K*\B_a\varphi) * \B_b\psi \doteq K*\B_b\psi
\]
Here, $K_1\dot=K_2$ means $K_1\subseteq K_2$ and $K_2\subseteq K_1$.

\subsubsection{DP2 -- Irrelevance of superseded beliefs}

If for two successive revisions, the second belief contradicts the first one, then the resulting beliefs for all agents should be the same as from revising according to the second belief only.
\[
   \entails\B_b\psi \rightarrow \B_a\neg\varphi\ \Rightarrow\ (K*\B_a\varphi) * \B_b\psi \doteq K*\B_b\psi
\]
 
\subsubsection{DP3 -- Consistency preservation across revisions}

If one agent's first-degree belief would be contained in the belief set after revision by any other first-degree belief, then in case the former is used to revise the belief set first, that first belief should be preserved through the second revision\/:
\[ \B_a\varphi\in K*\B_b\psi \ \Rightarrow\ \B_a\varphi \in (K*\B_a\varphi) * \B_b\psi \]

\subsubsection{DP4 -- Minimal change when reaffirming a belief}
After two consecutive revisions, the first agent should not end up believing the opposite unless they would do so if the belief set was revised by the second belief only\/:
\[ \B_a\neg\varphi\not\in K*\B_b\psi \ \Rightarrow\ \B_a\neg\varphi \not\in (K*\B_a\varphi) * \B_b\psi \]

\subsubsection{IN -- Independence}
After two consecutive revisions, the belief of the first agent should be preserved unless they would believe the opposite if the belief set was revised by the second belief only\/:
\[ \B_a\neg\varphi\not\in K*\B_b\psi \ \Rightarrow\ \B_a\varphi \in (K*\B_a\varphi) * \B_b\psi \]

Interestingly, while in the classical, single-agent case the Independence postulate \textbf{IN} strengthens both \textbf{DP3} and \textbf{DP4}~\cite{jin2007iterated} for any operator that satisfies the AGM postulates, in the generalized case Independence only strengthens \textbf{DP4}.
\begin{lemma}
\label{lem:independence}
Consider a belief revision operator that satisfies the generalized AGM postulates, then the operator satisfies multi-agent DP4 if it satisfies multi-agent Independence.
\end{lemma}
\begin{proof} 
Independence obviously implies DP4 unless $\{\B_a\phi,\B_a\neg\phi\}\subseteq (K*\B_a\varphi) * \B_b\psi$. The latter would mean that $a$ has inconsistent beliefs at the end of the two revisions, which according to Inconsistency~1 and~2 can only happen if~$\varphi\entails\bot$. This in turn implies $\B_a\neg\varphi\not\in K*\B_b\psi$, thus DP4 holds vacuously in this case also.
\end{proof}
It is noteworthy that \textbf{IN} would not entail \textbf{DP4} without the additional \textbf{Consistency~2} postulate, which guarantees that agent~$a$ does not end up believing in both $\varphi$ and $\neg \varphi$ as a result of further revision by $\B_b\psi$ after revising by $\B_a\varphi$.

\textbf{IN} does not imply \textbf{DP3} even if an operator satisfies all multi-agent AGM postulates, for the following reason\/: If $\B_a\bot\in K*\B_b\psi$ then \textbf{IN} vacuously holds while \textbf{DP3} is violated if $\B_a\bot\not\in (K*\B_a\bot\varphi)*\B_b\psi$. This is possible if an operator allows an agent to regain consistent beliefs when revising a multi-agent belief set by another agent's belief.

The next theorem summarizes the satisfaction of the generalized DP postulates of $\fm$.  
\begin{theorem}
\label{th:fm-dp}
$\fm$ satisfies \emph{\textbf{DP1}}, \emph{\textbf{DP3}}, and \emph{\textbf{DP4}} but not 
\emph{\textbf{DP2}} nor \emph{\textbf{
IN}}. However, $\fm$ does satisfy a weak version of \emph{\textbf{DP2}} where $\doteq$ is replaced by $\subseteq$.
\end{theorem}
\noindent
\textit{Proof.} \textbf{(DP2)}
Generalized full-meet does not satisfy \textbf{DP2} for the following reason\/: If $\B_b\psi$ is consistent with the current belief set~$K$ then $K\fm\B_b\psi=K+\B_b\psi$, hence by Lemma~\ref{lem:belief-other}, all agents retain all their beliefs when revising $K$ by $b$'s new belief. But if $K$ is revised by $\B_a\varphi$ first and $\models\B_b\psi\rightarrow\B_a\varphi$ holds then $(K\fm\B_b\varphi)\fm \B_b\psi=K_{(M_\emptyset,s)}+\B_b\psi$, which implies that the beliefs of all other agents have been erased.

The fact that full-meet does not satisfy \textbf{DP2} mirrors a result in classical, single-agent belief revision~\cite{jin2007iterated}. To show that full-meet multi-agent belief revision satisfies one direction of \textbf{DP2} (namely, that revising by a belief that is then superseded by a second, contradictory belief never introduces more beliefs than revision with the second belief directly) we make a case distinction.
\begin{enumerate}
    \item Suppose that $\models\varphi\leftrightarrow\bot$, then $\B_a\neg\varphi\in K$, hence $K\fm\B_a\varphi=K_{(M_\emptyset,s)}+\B_a\bot$, that is, there is no $a$-accessible world in~$s$ while all other agents' first-degree beliefs are tautological proposition formulas. Consequently, $(K\fm\B_a\varphi)\fm\B_b\psi\subseteq K\fm\B_b\psi$.
    \item Otherwise, $\models\B_b\psi\rightarrow\B_a\varphi$ implies $a=b$ and $\psi\models\neg\varphi$, hence $\varphi\models\neg\psi$. It follows that $\B_b\psi\not\in K\fm \B_a\varphi$, hence $(K\fm\B_a\varphi)\fm\B_b\psi=K_{(M_\emptyset,s)}+\B_b\psi$, which implies
    $(K\fm\B_a\varphi)\fm\B_b\psi\subseteq K\fm\B_b\psi$. \hfill{$\Box$}
\end{enumerate}

\section{Event Model-Based Belief Revision}
\label{sec:assessment}

While full-meet is an instructive operator to show that all generalized AGM and most DP postulates can be simultaneously satisfied, it is obviously too strong for practical purposes in  
that any (non-tautological) belief may be abandoned in case of a true revision. In this section, we introduce a more sophisticated operator for MBR based on event models. Event models (a.k.a.\ update models) have been used to describe transformations of epistemic states in multi-agent domains according to a predetermined transformation pattern~\cite{BaltagMS98}). This has also been used in modeling belief-altering actions in high-level action languages \cite{BaralGPS22,RajaratnamT21}. Nonetheless, event models proposed for the purpose of updating  beliefs of agents after an action occurrence are not suitable for revising the beliefs of agents. This is illustrated in Fig.~\ref{fig:motivation-event-model}\/:
\begin{figure}[t]
    \centering
    \includegraphics[width=0.95\linewidth]{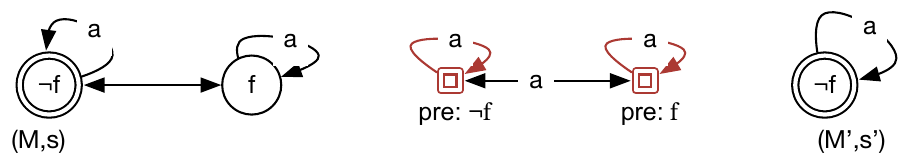}
    \caption{Revising with $\B_a f$ is not the same as sensing $f$}
    \label{fig:motivation-event-model}
\end{figure}
In the model $(M,s)$ (left), simplified to consider only a single agent, $a$ does not believe $f$ nor $\neg f$. The event model for sensing $f$ is in the middle, whose update on $(M,s)$ is $(M',s')$ (right), which indicates that $a$ believes $\neg f$ after sensing $f$. Such an update does not allow for the revision of $(M,s)$ with $\B_a f$.

For the discussion of the generalized postulates, we need some extra notation.
A \emph{literal\/} is either a proposition $p \in \calp$ or its negation $\neg p$.  
For a literal $\ell$, $\neg \ell$ denotes its negation, with $\neg \neg p = p$ for $p \in \calp$.   
An \emph{inference rule\/} (or rule) among literals is of the form $\lambda \rightarrow \delta$  where $\lambda $ and $\delta$ are sets of literals. 
Given a set of literals $w$ and a set of rules $R$, $C_R(w)$ denotes the minimal set (w.r.t.~$\subseteq$) of literals $w'$ such that $w \subseteq w'$ and, for every $\lambda  \rightarrow \delta$ in $R$, if $w \models \lambda $ then $ w' \models \delta$.  

In the presence of the set of rules $R$, the valuation function $\pi$ of any Kripke model $M$ must satisfy $R$ as well. Therefore, we require that for every world $u$ in $M$, $\pi(u)$ is consistent, i.e., $\pi(u) = C_R(\pi(u))$. 
For a set of literals $\varphi$ and an interpretation $u$, 
let $u \star \varphi$ denote an interpretation $u'$ such that $u' = C_R((u \cap u') \cup \varphi)$.  
In general, there might exist several interpretations $u'$ satisfying the aforementioned equation. Fortunately, 
it is well-known that there are conditions on $R$  such that there is a unique $u'$ that satisfies the equation  (see, e.g., \citeauthor{TuSGM11}'s (\citeyear{TuSGM11}) work). In this paper, we will assume that $R$ only consists of rules of the form $p \rightarrow q$ where $p, q$ are literals and, for each $u$ and $p$, there is a unique $u' = C_R((u \cap u') \cup \{p\})$, 
since this is sufficient for dealing with the \textbf{DP1} and \textbf{DP2} postulates. 
    
Consider $a \in \calag$ and a set of literals $\varphi$. 
We define $\mathbf{\Sigma}^a(\varphi)$, called the \emph{event model for revision by $\B_a \varphi$}, as the event model   
 $\langle \Sigma, \{E_a\}_{a \in \calag}, pre, \mathit{eff}  \rangle$   
 where 
\begin{itemize}
    \item $\Sigma = \{\sigma, \delta, \sigma_a, \delta_a, \epsilon\}$;
    \item $E_a = \{(\sigma,\sigma_a), (\sigma_a,\epsilon), (\delta,\delta_a), (\delta_a,\epsilon), (\epsilon,\epsilon)\}$;
    \item $E_x = \{(\eta,\epsilon) \mid \eta \in \Sigma\setminus\{\sigma_a,\delta_a\}\}$ for $x \in \calag \setminus \{a\}$; 
    \item $pre(\sigma) = \neg \B_a \neg \varphi$, 
    $pre(\delta) = \B_a \neg \varphi$, 
    $pre(\sigma_a) = \varphi$, 
    $pre(\delta_a) = \neg \varphi$, 
    and $pre(\epsilon) = \top$.
    \item for $u \in W$, $\mathit{eff}(u,\eta) = \pi(u)$ for $\eta \in \Sigma \setminus \{\delta_a\}$,  
    and $\mathit{eff}(u,\delta_a) = \pi(u) \star \varphi$.  
\end{itemize}
In the above definition, $\Sigma$ is the set of events representing possible views of the event ``\emph{$(M,s)$ is revised by $\B_a \varphi$}''. Intuitively, this revision could affect the worlds accessible by $a$ in the following ways\/: If $\neg \B_a \neg \varphi$ is true, i.e., 
$\B_a \varphi$ or $\neg (\B_a \varphi \wedge \B_a \neg \varphi)$ holds, then some worlds accessible by $a$ satisfy $\varphi$ and $a$ can just eliminate all worlds satisfying $\neg \varphi$ from its accessibility. 
\begin{figure}[t]
    \centering
    \includegraphics[width=0.69\linewidth]{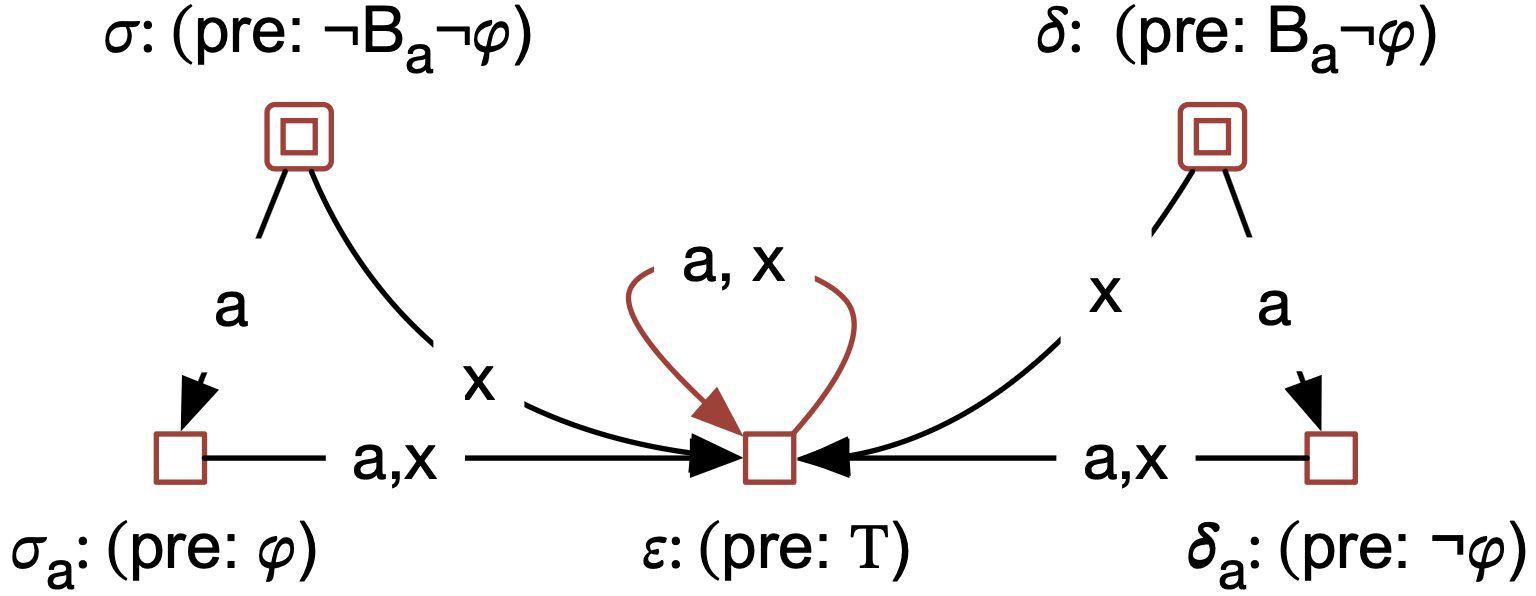}
    \caption{Revision event model for $\B_a \varphi$}
    \label{fig:revision-operator}
\end{figure}
This is represented by the events $\sigma$ and $\sigma_a$ (Fig.~\ref{fig:revision-operator}) where $\sigma$ denotes a \emph{designated event} and $\sigma_a$ has the precondition $\varphi$.  
If, on the other hand, 
$\B_a \neg \varphi$ is true, then agent~$a$ should revise his beliefs in every world accessible by~$a$, via the definition of $\mathit{eff}$.  This is represented by the events $\delta$ and $\delta_a$.
For agent $x \in \calag \setminus \{a\}$, nothing changes and thus, $x$'s view is that only the event $\epsilon$ occurs. 

The revision of $\B_a \varphi$ in $(M,s)$,
denoted by $(M,s) \ev \B_a \varphi$, 
is a pointed Kripke structure 
$(M^\varphi,s^\varphi)$ where $M^\varphi = (W^\varphi, \{R^\varphi_x\}_{x \in \calag}, \pi^\varphi)$
and 
\begin{itemize}
    
\item 
$W^\varphi = \{(u,\tau) \mid u \in W, \tau \in \Sigma, (M,u) \models pre(\tau)\}$; 

\item 
$((u,\tau), (u',\tau')) \in R^\varphi_x$ iff $(u,\tau),(u',\tau')\in W^\varphi$ along with
$(u,u') \in R_x$ and $(\tau,\tau') \in E_x$;  

\item $\pi^\varphi((u,\tau)) = \mathit{eff}(u, \tau)$; and 

\item $s^\varphi = (s, \sigma)$ if  $(M,s) \models \neg \B_a \neg \varphi$;
otherwise, $s^\varphi = (s, \delta)$.  

\end{itemize}

Before we discuss the properties of $\ev$ in detail, let us observe that
if $R_a(s) \ne \emptyset$ then $R^\varphi_a(s^\varphi) \ne \emptyset$, i.e.,  
$a$ has consistent belief in $(M^\varphi, s^\varphi)$. This is the main 
difference between $\ev$ and $\fm$. 
We prove that beliefs of other agents do not change. 
\begin{lemma}
\label{lem:revision-other}
    For $x \in \calag \setminus \{a\}$ and a proposition formula $\psi$,  
    $ (M^\varphi,s^\varphi) \models \B_x \psi$ iff $(M,s) \models \B_x \psi$.
\end{lemma}
\begin{proof}
    By the construction of $(M^\varphi,s^\varphi)$, we have that 
    $s^\varphi = (s,\eta)$ such that $(M,s) \models pre(\eta)$.
    Furthermore, $(u, \epsilon) \in W^\varphi$ for every $u \in W$ and  
    $((s, \eta), (u, \epsilon)) \in R^\varphi_x$ iff $(s, u) \in R_x$. 
    Because $\pi^\varphi((u,\epsilon)) = \pi(u)$ for every $u \in W$,   
    $ (M^\varphi,s^\varphi) \models \B_x \psi$ iff $(M,s) \models \B_x \psi$.
\end{proof}
\begin{theorem}
\label{th:ev-agm}
$\ev$ satisfies the generalized AGM postulates.
\end{theorem}
\noindent
\emph{Proof.}
    As above, proofs are given for selected interesting postulates only. We assume a consistent $K$ that is represented by $(M,s)$ with  
 $M = (W, \{R_x\}_{x \in \calag}, \pi)$. 
We use $(M,s) \ev \B_a \varphi$ and 
$(M^\varphi,s^\varphi)$ interchangeably  
where $M^\varphi = (W^\varphi, \{R^\varphi_x\}_{x \in \calag}, \pi^\varphi)$.
Successive revisions such as $(K \ev \B_a \varphi) \ev \B_a \psi$ 
will be denoted by $(M^{\varphi; \psi}, s^{\varphi; \psi})$ etc. 
Furthermore, $R_a(s|\varphi) = \{u \mid u \in R_a(s), \pi(u) \models \varphi\}$.
Below, we use conjunctions of literals and sets of literals interchangeably, and we always assume that a conjunction of literals $\varphi$ is consistent in the classical sense, i.e., $\varphi \not\vdash \bot$. 
For a set of literals $\varphi$,
$\neg \varphi$ denotes $\{\neg \ell \mid \ell \in \varphi\}$. 

(\emph{Success}) We show $K\ev{} \B_a\varphi \models \B_a \varphi$ via two cases: 
\begin{enumerate}
    \item $(M,s) \models \neg \B_a \neg \varphi$, i.e., 
    $a$ does not believe in $\neg \varphi$ before the revision.
    This implies that $s^\varphi = (s, \sigma)$. 
    Furthermore, because $K$ is consistent, we have that $R_a(s) \ne \emptyset$
    and there exists some $u \in W$ such that $(s,u) \in R_a$ and
    $\pi(u) \models \varphi$, which implies that $((s,\sigma), (u,\sigma_a)) \in R^\varphi_a$, 
    i.e., $R^\varphi(s^\varphi) \ne \emptyset$. 
    In addition, if $((s, \sigma), (u, \rho)) \in R^\varphi_a$ then $\rho = \sigma_a$, and hence, $\pi(u) \models \varphi$, because 
    $E_a$ contains only one element related to $\sigma$, $(\sigma, \sigma_a)$, and $(s,u) \in R_a$.
    Thus, we have that $(M^\varphi, s^\varphi) \models \B_a \varphi$. 

    \item $(M,s) \models \B_a \neg \varphi$.
    This implies that $s^\varphi = (s, \delta)$. 
    Again,  because $K$ is consistent, we have that $R_a(s) \ne \emptyset$
    and for each $u \in R_a(s)$, $\pi(u) \models \neg \varphi$.  
    By the construction of $(M^\varphi, s^\varphi)$, 
    $R^\varphi_a(s^\varphi) \ne \emptyset$. 
    Consider $u' \in R^\varphi_a(s^\varphi)$. We have that 
    $u ' = (u, \delta_a)$ for some $u \in R_a(s)$, and hence, $\pi^\varphi(u') = C_R((\pi(u) \cap \pi^\varphi(u')) \cup \{\varphi\})$, which implies $\varphi \in \pi^\varphi(u')$.  
    It follows that $(M^\varphi, s^\varphi) \models \B_a \varphi$.    
    \hfill{$\Box$}
\end{enumerate}

\emph{(Superexpansion)}  Consider a propositional formula $\phi$, then Lemma~\ref{lem:revision-other} and 
Lemma~\ref{lem:belief-other} imply that 
to prove that $\ev$ satisfies this postulate, 
it suffices to show that  
$ K \ev{}  \B_a(\varphi \wedge \psi) \models \B_a \phi$ implies 
$(K \ev{}  \B_a\varphi) + \B_a \psi  \models \B_a \phi$.
Let $(M',s')$ denote $(M^{\varphi},s^{\varphi}) + \B_a \psi$. 
We consider two cases\/:
\begin{itemize}
    \item $(M,s) \models \neg \B_a \neg (\varphi \wedge \psi)$.
    So, $s^{\varphi \wedge \psi} = (s, \sigma^{\varphi \wedge \psi})$ and $(s^{\varphi \wedge \psi},u^{\varphi \wedge \psi}) \in R^{\varphi \wedge \psi}_a$ 
    where $u^{\varphi \wedge \psi} = (u,\sigma^{\varphi \wedge \psi}_a)$ iff $(s,u) \in R_a$ and $\pi(u) \models (\varphi \wedge \psi)$.
    Since $(M,s) \models \neg \B_a \neg (\varphi \wedge \psi)$, we have 
    that $(M,s) \models \neg \B_a \neg \varphi$, i.e., $u \in R_a(s|\varphi \wedge \psi)$. 
    Thus $s^{\varphi} = (s, \sigma^{\varphi})$ 
        and $(s^{\varphi},u^{\varphi}) \in R^{\varphi}_a$ where $u^{\varphi} = (u,\sigma^{\varphi}_a)$ iff $(s,u) \in R_a$ and $\pi(u) \models \varphi$, i.e., $u \in R_a(s|\varphi)$.
        Because $R_a(s|\varphi) = R_a(s|\varphi \wedge \psi) \cup R_a(s|\varphi \wedge \neg \psi)$ and 
        the construction of $(K \ev{}  \B_a\varphi)   \models \B_a \phi$ we have 
        $R'_a(s') = \{u^r \mid u \in R_a(s), \pi(u) \models \varphi \wedge \psi\}$, which proves
        the consequence of the postulate.
        
    \item $(M,s) \models \B_a \neg (\varphi \wedge \psi)$. Then, $s^{\varphi \wedge \psi} = (s, \delta^{\varphi \wedge \psi})$ and 
    for every $u^{\varphi \wedge \psi}$ such that     
    $(s^{\varphi \wedge \psi},u^{\varphi \wedge \psi}) \in R^{\varphi \wedge \psi}_a$ holds, we have 
    $u^{\varphi \wedge \psi} = (u,\delta^{\varphi \wedge \psi}_a)$ iff 
    $(s,u) \in R_a$.
    Therefore, we can conclude that $K \ev{}  \B_a(\varphi \wedge \psi) \models \B_a \phi$ iff  for every $u 
    \in R_a(s)$, $\pi(u) \star (\varphi \wedge \psi) \models \phi$.
    There are two cases: 
    \begin{enumerate}
        \item $(M,s) \models \B_a \neg \varphi$. 
        In this case, similar arguments to the above allow us to conclude that 
        $(s^{\varphi},u^{\varphi}) \in R^{\varphi}_a$ iff 
        $u^{\varphi} = (u,\delta^{\varphi}_a)$ for some $(s,u) \in R_a$ and $\pi(u^\varphi) = \pi(u)\star \varphi$.
        Because of the construction of $(K \ev{}  \B_a\varphi) + \B_a\psi$, we have two sub-cases: 
        \begin{enumerate}
        \item $K \ev{}  \B_a\varphi \models \B_a \neg \psi$. In this case, 
        $a$ has inconsistent beliefs in $(K \ev{}  \B_a\varphi) + \B_a\psi$, i.e.,
        the postulate holds trivially. 
        \item $K \ev{}  \B_a\varphi \models \neg \B_a \neg \psi$.  
        In this case, we can conclude that   $R'_a(s') \ne \emptyset$. 
        Furthermore, $(s',u') \in R'_a$ iff $u' = u^\varphi$ and 
        $\pi(u^\varphi) \models \psi$. 
        Given that $\star$ satisfies the classical AGM postulate of Superexpansion, we can conclude that   
        $K \ev{}  \B_a(\varphi \wedge \psi) \models \B_a \phi$ 
        implies 
        $K \ev{}  \B_a\varphi + \B_a \psi \models \B_a \phi$.
        \end{enumerate}
                
        \item $(M,s) \models \neg \B_a \neg \varphi$. 
        We have $s^{\varphi} = (s, \sigma^{\varphi})$ and $(s^{\varphi},u^{\varphi}) \in R^{\varphi}_a$ where $u^{\varphi} = (u,\sigma_a^{\varphi})$ 
        iff $(s,u) \in R_a$ and $\pi(u) \models \varphi$, which implies that   $\pi^{\varphi}(u^{\varphi}) \models \neg \psi$ 
        for every $(s^{\varphi},u^{\varphi}) \in R^{\varphi}_a$. This means that $a$'s belief is inconsistent in  
        $(K \ev{}  \B_a\varphi) + \B_a\psi$, and thus, the postulate also holds in this case. \hfill{$\Box$}
    \end{enumerate}
\end{itemize}

\noindent
While $\ev$ satisfies all generalized AGM postulates, it does not satisfy \textbf{DP1}-\textbf{DP2}. This is because $\ev$ attempts to remove uncertainty in the beliefs of an agent before revising them, which is  encoded in the event $\sigma$ in which the agent is uncertain about $\varphi$ and,  due to the precondition of $\sigma_a$, the operator only retains worlds satisfying $\varphi$. Nevertheless, 
these postulates, in a simpler form where $\varphi$ and $\psi$ are literals, 
are satisfied under restricted conditions. 
\begin{theorem}
\label{th:ev-dp}
 $\ev$ satisfies (\emph{i}) 
 \emph{\textbf{DP1}} 
 if $\B_a \neg p \not\in K \ev{} \B_a q$; 
 (\emph{ii})  
 \emph{\textbf{DP2}} if $\B_a q \vee \B_a p \in K$;
 (\emph{iii})
 \emph{\textbf{DP3}}; 
 (\emph{iv}) \emph{\textbf{DP4}} and \emph{\textbf{IN}} if agent $a$ has consistent beliefs in $K$.
\end{theorem}
\noindent
\emph{Proof.}
\textbf{(DP1)} 
Since $ \models \B_b  p \rightarrow \B_a  q$, we have that 
$ p \rightarrow  q$ and $a$ and $b$ are identical. 
Therefore, we will show that $\ev$ satisfies this postulate by proving that 
if $\models \B_a p \rightarrow \B_a q$ and $K \ev{} B_a q \not\models \B_a \neg p$ hold then  
$(K \ev{} \B_a q) \ev{}  \B_a p \doteq K\ev{} \B_a p$. 
\begin{enumerate}
    \item $(M,s) \models \neg \B_a \neg  q$. 
    So, if $(s^ q,u^ q) \in R^ q_a$ then  
    $u \in R_a(s)$ and $\pi(u) \models   q$ and 
    $\pi^q(u^q) = \pi(u)$. 
    Since $(M,s) \ev \B_a  q \models \neg \B_a \neg  p $, for $u^{ q; p} \in W^{ q; p} $ such that 
        $(s^{ q; p}, u^{ q; p}) \in R^{ q; p}_a$     
        iff $(s^{ q}, u^{ q}) \in R^{ q}_a$ and    
        $\pi^{ q}(u^{ q}) \models  p$.
        This implies that $u^{ p} \in W^ p$ and $(s^ p, u^ p) \in R^ p_a$. 
        
        $(M,s) \ev \B_a  q \models \neg \B_a \neg  p $ also 
        implies that $(M,s) \models \neg \B_a \neg  p$. 
        Therefore, for every $u\in R_a(s)$ such that $u^{ p} \in W^ p$ and $(s^ p, u^ p) \in R^ p_a$,
        we can conclude that  $(s^{ q; p}, u^{ q; p}) \in R^{ q; p}_a$ because $\models  p \rightarrow  q$. 
        The above imply that  
$(K \ev{} \B_a q) \ev{}  \B_a p \doteq  K\ev{} \B_a p$ for this case.

    \item $(M,s) \models \B_a \neg  q$. 
    This implies that $(M,s) \models \B_a \neg  p $ and $s^ q = (s, \delta^ q)$ 
    and $(s^ q,u^ q) \in R^ q_a$ iff $u^ q = (u,\delta^ q_a)$ and 
    $\pi^ q(u^ q) = \pi(u) \star  q$. By the definition of $\star$, we can show that 
    $(M,s) \models \B_a \neg  p$ implies 
    $\pi^ q(u^ q) \models \neg  p$. 
    This implies that $s^{ q; p} = ((s, \sigma^ q), \delta^{ q; p})$ and 
    for every $u \in R_s(a)$,
     $(s^{ q; p}, u^{ q; p}) \in R^{ q; p}_a$, 
     $(s^ q, u^ q) \in R^ q_a$, and 
    $\pi^{ q; p}(u^{ q; p}) = (\pi(u) \star  q) \star  p$.
    On the other hand, by construction of $(M,s) \ev{} p$,   
    $s^ p = (s, \delta^ p)$ 
    and, for every $u \in W$, 
    $(s^ p,u^ p) \in R^ p_a$ where $u^ p = (u,\delta^ p_a)$ and 
    $\pi^ p(u^ p) = \pi(u) \star  p$. Again by definition of $\star$, we have that 
    $\pi(u) \star  p = (\pi(u) \star  q) \star p$ for every $u \in R_a(s)$.
    This proves that  
$(K \ev{} \B_a q) \ev{}  \B_a p \doteq  K\ev{} \B_a p$ for this case. \hfill{$\Box$}
\end{enumerate}

We conclude this section with the introduction of an event-based revision operator, denoted by $\fb$, that satisfies all the generalized postulates if the underlying classical revision operator satisfies the classical postulates. The operator 
$\fb$ differs from $\ev$ in that it does not attempt to remove the agent's uncertainty before revision. It 
employs the event model $$\mathbf{\Sigma}^a_b(\varphi) =  \langle \Sigma, \{E_a\}_{a \in \calag}, pre, \mathit{eff}  \rangle$$ 
 where $\Sigma = \{\sigma, \sigma_a, \epsilon\}$ and  
\begin{itemize}
    \item $E_a = \{(\sigma,\sigma_a), (\sigma_a,\epsilon),  (\epsilon,\epsilon)\}$;  
    \item $E_x = \{(\sigma,\epsilon), (\sigma_a,\epsilon),  (\epsilon,\epsilon)\}$ 
 for $x \in \calag \setminus \{a\}$; 
    \item $pre(\eta) =  \top$ for every $\eta \in \Sigma$; and 
    \item for $u \in W$, $\mathit{eff}(u,\eta) = \pi(u)$ for $\eta \in \Sigma \setminus \{\sigma_a\}$,  
and $\mathit{eff}(u,\sigma_a) = \pi(u) \star \varphi$. 
\end{itemize}
$(M^\varphi, s^\varphi)$, the result of $(M,s) \fb \B_a \varphi$, is defined similar to $(M,s) \ev \B_a \varphi$ with two changes: 
(\emph{i}) $\mathbf{\Sigma}^a_b(\varphi)$ is used instead of $\mathbf{\Sigma}^a(\varphi)$ and (\emph{ii}) $s^\varphi = (s, \sigma)$. $\fb$ satisfies every postulate. 
\begin{theorem}
\label{th:fb-all}
    $\fb$ satisfies all MBR postulates. 
\end{theorem}

\noindent
As a short proof sketch for this theorem, we observe the following properties of $(M,s) \fb \varphi$: 
For every $u \in W$ and $\lambda \in \Sigma$, $u^\varphi = (u, \lambda) \in W^\varphi$. 
Furthermore, $(s^\varphi, u^\varphi) \in R^\varphi_a$ iff $(s,u) \in R_a(s)$  and $\pi^\varphi(u^\varphi) = \pi(u) \star \varphi$. The conclusion of the theorem then relies on the following observations:  
\begin{itemize}
    \item  for every $u$, $\pi^\varphi(u^\varphi) = \pi(u) \star \varphi \models \varphi$; and
    \item for every $u$, if $\pi(u) \star \varphi \not\models \neg \psi$ then 
    $\pi(u) \star  \varphi \wedge \psi  = (\pi(u) \star \varphi) + \psi$,  
    and if $\pi(u) \star \varphi \models \neg \psi$ then $(\pi(u) \star \varphi) + \psi$ is inconsistent. 
\end{itemize}

\section{Discussion and Outlook}

We defined the problem of multi-agent belief revision (MBR) based on a single Kripke model as commonly used in epistemic reasoning about actions and planning, and we proposed a generalization of the AGM and DP postulates to this multi-agent case. We identified the challenges faced by the task of constructing MBR operators that adhere to all generalized postulates including those for iterated revision, and we presented results with a generalized full-meet operator and an event-based revision operator that satisfy most but not all of them. 

Theorem~\ref{th:fm-dp} indicates that our generalized full-meet operator $\fm$ does not satisfy \textbf{DP2}. 
This is because $K \fm \B_a \varphi$ may include $\B_b \neg \psi$, and then further revision by $\B_b\psi$ erases all beliefs other than the new one, hence the conclusion of the postulate does not hold.
For this reason, it is generally challenging to construct belief revision operators that satisfy \textbf{DP2}. 
Looking at the definition of $\fm$, it is obvious that the first challenge is the problem of dealing with false beliefs, i.e., defining $K * \B_a \varphi$ given that $\B_a \neg \varphi \in K$. 
The second challenge is related to the revision of the accessibility relation of $a$. More precisely, let us 
denote with $R_a|\eta$ the set $\{(s, u) \in R_a \mid \pi(u) \models \eta\}$ for $\eta \in \call_\calp$. 
Operator $\fm$ essentially removes $R_a| \neg \varphi$ from consideration in constructing $K \fm \B_a \varphi$. Thus, if $R_a |\neg \varphi \ne \emptyset$ then the set of worlds accessible by $a$ in $K \fm \B_a \varphi$   
is a proper subset of $R_a$,   the  set of worlds accessible by $a$ in $K$, 
which prevents     
$\fm$ from satisfying \textbf{DP2}. 

The above discussion implies that for a belief revision operator to satisfy the generalized DP postulates, it must deal with false beliefs as well as the accessibility relations of agents \emph{properly}. We expect that ideas from studies of reasoning with false beliefs in dynamic epistemic logic \cite{DitmarschHV20} or using update models \cite{son2024dealing} can be helpful in addressing the first issue. To deal with the second issue, approaches based on a pre-order relation between Kripke models (see, e.g., the works by \citeauthor{aucher:genera}~(\citeyear{aucher:genera}) and \citeauthor{DitmarschHKB07}~(\citeyear{DitmarschHKB07})) might be necessary.  
It is worth noting that the premise of \textbf{DP2}, $\entails \B_b \psi \rightarrow \B_a \neg \varphi$, could also be considered as a culprit for this postualte not to be satisfied by $\fm$. Modifying the premise to $\B_b \psi \rightarrow \B_a \neg \varphi\,\in\, K$, and also requiring that $a \ne b$, could be a viable alternative to generalizing \textbf{DP2} worthy of consideration.

Turning to the issue of higher-degree beliefs, the generalized revision postulates presented in this paper focus on first-degree beliefs, which are arguably the most fundamental ones, and postulating anything about the retention or revision of higher-degree beliefs seems more difficult to justify. But of course the concrete revision operators discussed in this paper all do precisely define how all higher-degree beliefs, including common beliefs, are changed upon revision, too. It is therefore an interesting direction for future work to develop further generalizations of the AGM and DP postulates to higher-order beliefs, with the aim of classifying concrete multi-agent belief revision operators according to their treatment of each nested belief.
 
It is worth noting that the strong relationship between MBR and \emph{announcements\/}, a type of actions specific to multi-agent domains and very important for epistemic planning, raises several interesting problems worthy of consideration, too, especially since higher-degree beliefs of agents are 
an important subject of study in formalizing announcement actions in the literature.
Several approaches to dealing with announcements have been proposed, but the majority of them place restrictions on the announcements (e.g., only considering public announcements) and 
the relationship between different methods is hardly understood. 
A systematic evaluation of these approaches under the lens of a set of generalized postulates for MBR could provide insights into the development of a general approach to dealing with announcements.  
In this regard, we note that \citeauthor{BaltagS05} (\citeyear{BaltagS05}) also discuss belief revision in the multi-agent case, as we did in this paper. Our approach differs in two key respects, however: First, they employ epistemic plausibility models as the underlying representation, whereas we adopt Kripke structures as used in dynamic epistemtic reasoning and planning. Second, an epistemic plausibility model contains a priori plausibility relations for agents that dictate how agents would revise their beliefs. In our approach, agents do not have such policies. A detailed comparison between the two approaches  will be one of our interesting research topics in the near future. Similarly, \citeauthor{LoriniPS22} (\citeyear{LoriniPS22}) discuss belief updates for epistemic actions over belief bases, and thus could also be used for agents to revise their beliefs when an announcment is made. However, their focus is to define the updates. It would be interesting to determine whether the result of the update by a private announcement ``\emph{agent $a$ was told that $\varphi$}'' satisfies the generalized postulates discussed in this paper, similar to our $\ev$ operator. Again, we leave this for future work.

\section*{Acknowledgements}

This research was partially supported by the Australian Research Council (grant \#DP250101822) and by the Australian Government through the CRC-P project \emph{Urban Copilot -- AI accelerating where we build Australia's future\/}.
The second author acknowledges the partial support of the NSF grants \#2139028, \#2139028, \#2151254, and the internal IAAM grant \#139198. 
 
\section*{AI Declaration}

AI tools were used solely for automated spelling and grammar checks and to provide stylistic suggestions.


\ifarxive

\else

\fi

\ifarxive
\appendix

\clearpage

\appendix

\begin{centerline}{\LARGE\bf Proofs}\end{centerline}

\section*{Full Meet and Generalized AGM Postulates}

For simplicity of the presentation, we include the proofs of all items in this document. 
As such, some proofs presented in the main body of the paper are repeated here. 

\subsubsection{Closure} By our definition, it is clear that  
\[ K \fm \B_a\varphi = \Cn(K \fm \B_a\varphi).\] 

\subsubsection{Success}
By definition, if $(M',s')$ is the result of expanding a belief set by $\B_a\varphi$, then $\pi(v)\models\varphi$ for all $v^r\in W'$ such that $(s',v^r)\in R'_a$. This implies $K \fm \B_a\varphi \models \B_a \varphi$.

\subsubsection{Inclusion}
\begin{enumerate}
\item Suppose that $\B_a\neg\varphi\not\in K_{(M,s)}$, then the definition of~$\fm$ implies that $K_{(M,s)}\fm\B_a\varphi\subseteq K_{(M,s)}+\B_a\varphi$.

\item If $\B_a\neg\varphi\in K_{(M,s)}$, then
\begin{enumerate}
    \item $R'_a(s')=\emptyset$ where $(M',s')=(M,s)+\B_a\varphi$, hence $K+\B_a\varphi\models\B_a\bot$ by Lemma~\ref{lem:belief-agent}, which implies that $K_{(M,s)}+\B_a\varphi\models\B_a\psi$ for any proposition formula~$\psi$;
    \item for $x\in\calag\setminus\{a\}$, $(M_\emptyset,s)\models\B_x\psi$ if, and only if, \mbox{$\models\psi$}; by Lemma~\ref{lem:belief-other} this holds for $(M_\emptyset,s)+\B_a\varphi$ as well, hence $K_{(M_\emptyset,s)}+\B_a\varphi\models\B_b\psi$ only for tautologies~$\psi$.
\end{enumerate}
Taken together, $K_{(M_\emptyset,s)}+\B_a\varphi\subseteq K_{(M,s)}+\B_a\varphi$, hence $K_{(M,s)}\fm\B_a\varphi\subseteq K_{(M,s)}+\B_a\varphi$.
\end{enumerate}

\subsubsection{Vacuity}
This holds by definition.

\subsubsection{Consistency 1}
Suppose $\B_a\varphi$ is consistent, i.e., $\not\models\varphi\rightarrow\bot$.
\begin{enumerate}
\item If $\B_a\neg\varphi\not\in K_{(M,s)}$, then $(s',v^r)\in R'_a$ for some $v^r\in W'$ such that $\pi[v^r]\models\varphi$, where $(M',s')=(M,s)+\B_a\varphi$. Hence, agent~$a$ has cosistent beliefs in $K_{(M,s)}\fm\B_a\varphi$.

\item If $\B_a\neg\varphi\in K_{(M,s)}$, then $(s',v^r)\in R'_a$ for some $v^r\in W'$ such that $\pi[v^r]\models\varphi$, where $(M',s')=(M_\emptyset,s)+\B_a\varphi$. Hence, agent~$a$ has cosistent beliefs in $K_{(M,s)}\fm\B_a\varphi$.
\end{enumerate}

\subsubsection{Consistency 2}
If $\B_a\varphi$ is consistent with $K$ then the beliefs of any agent other than~$a$ do not change as a result of expanding $(M,s)$ by $\B_a\varphi$. If $\B_a\varphi$ is inconsistent, then any agent other than~$a$ has tautological beliefs only, hence does not have inconsistent beliefs.

\subsubsection{Extensionality} If $\models \varphi\leftrightarrow\psi$ then $(M,s)+\B_a\varphi$ and $(M,s)+\B_a\psi$ are identical Kripke structures for any $(M,s)$ by definiton.

\subsubsection{Superexpansion}
\begin{enumerate}
   \item Suppose $\B_a\neg(\varphi\wedge\psi)\not\in K$. Since $K$ is deductively closed, it follows that $\B_a\neg\varphi\not\in K$.
   Hence, agent~$a$ has consistent beliefs in $K+\B_a(\varphi\wedge\psi)$, in $K+\B_a\varphi$, and in $(K+\B_a\varphi)+\B_a\psi$.
   By Lemma~\ref{lem:first-degree-belief-agent} it follows that $(M,s)+\B_a(\varphi\wedge\psi)\models\B_a\chi$ iff $\chi\in\Cn(\{\varphi\wedge\psi\}\cup\{\phi\in\calb_{\calag,\calp} \mid\ \B_a\phi\in K\})$. This is equivalent to
   $\chi\in\Cn(\{\psi\}\cup\Cn(\{\varphi\}\cup\{\phi\in\calb_{\calag,\calp} \mid\ \B_a\phi\in K\}))$, which in turn by Lemma~\ref{lem:first-degree-belief-agent} is equivalent to
   $\left(K+\B_a\varphi\right) + \B_a\psi\models\B_a\chi$.
      By Lemma~\ref{lem:belief-other} it follows that all other agents too have the same beliefs in $K+\B_a(\varphi\wedge\psi)$ and in $(K+\B_a\varphi)+\B_a\psi$.
      We conclude that  $K \fm \B_a(\varphi \wedge \psi) \subseteq (K \fm \B_a\varphi) + \B_a\psi$.

   \item Suppose $\B_a\neg(\varphi\wedge\psi)\in K$. By definition of $(M_\emptyset,s)+\B_a(\varphi\wedge\psi)$ and Lemma~\ref{lem:first-degree-belief-agent} it follows that, for any first-degree belief, $\B_a\phi\in K\fm \B_a(\varphi\wedge\psi)$ iff $\phi\in\Cn(\varphi\wedge\psi)$, and for $x\in\calag\setminus\{a\}$, $\B_x\phi\in K\fm \B_a(\varphi\wedge\psi)$ iff $\models\phi$.
   \begin{enumerate}
      \item Suppose $\B_a\neg\varphi\not\in K$, then $(K\fm\B_a\varphi)+\B_a\psi=(K+\B_a\varphi)+\B_a\psi$. From $\B_a\neg(\varphi\wedge\psi)\in K$ it follows that $\B_a\neg\psi\in K+\B_a\varphi$. Hence, $a$ has inconsistent beliefs in $(K+\B_a\varphi)+\B_a\psi$ while the beliefs of all other agents $x\in\calag\setminus\{a\}$ are the same in $(K+\B_a\varphi)+\B_a\psi$ and $K+\B_a(\varphi\wedge\psi)$. 
      \item Suppose $\B_a\neg\varphi\in K$, then $(K\fm\B_a\varphi)+\B_a\psi=(K_{(M_\emptyset,s)}+\B_a\varphi)+\B_a\psi$.
      Hence, $a$'s first-degree beliefs in $(K\fm\B_a\varphi)+\B_a\psi$ are exactly the logical consequences of $\varphi\wedge\psi$ while the beliefs of all other agents $x\in\calag\setminus\{a\}$ are the same in $(K\fm \B_a\varphi)+\B_a\psi$ and $K+\B_a(\varphi\wedge\psi)$.
   \end{enumerate}
   Taken together, $K \fm \B_a(\varphi \wedge \psi) \subseteq (K \fm \B_a\varphi) + \B_a\psi$. 
\end{enumerate}

\subsubsection{Subexpansion}
The proof is identical to cases~1 and 2(b) for superexpansion.

\section*{Full-Meet and Generalized DP Postulates}

\subsubsection{DP1 -- Successive revision respect}
Suppose $\B_a\psi$ is consistent with~$K$, then so is $\B_a\varphi$ since $\psi\models\varphi$. It follows that $K\fm\B_a\psi = K_{(M,s)+\B_a\psi}$ and $(K\fm\B_a\varphi)\fm\B_a\psi=K_{((M,s)+\B_a\varphi)+\B_a\psi}$.
The two Kripke structures are bisimilar except for unreachable worlds, hence they entail the same set of belief formulas\/:
\begin{itemize}
    \item $(M_\varphi,s_\varphi)=(M,s)+\B_a\varphi$ contains the worlds~$W$ in~$M$ plus a replicate~$W_\varphi$ of these worlds such that a world is $a$-accessible  from~$s_\varphi$ if, and only if, it is in $W_\varphi$ and satisfies~$\varphi$.
    $(M',s')=(M_\varphi,s_\varphi)+\B_a\psi$ contains the worlds $W\cup W_\varphi$ plus a replicate of these, $W'\cup W'_\varphi$, such that a world is $a$-accessible from~$s'$ iff it is in $W'\cup W'_\varphi$ and satisfies~$\varphi\wedge\psi$, which is equivalent to saying it satisfies~$\psi$ since $\psi\models\varphi$; while none of the worlds in~$W_\varphi$ are reachable from~$s'$ because the only links from~$s_\varphi$ into worlds in~$W_\varphi$ are labeled with agent~$a$.
    \item $(M_\psi,s_\psi)=(M,s)+\B_a\psi$ contains the worlds~$W$ in~$M$ plus a replicate~$W_\psi$ of these worlds such that a world is $a$-accessible from~$s_\psi$ iff it is in $W_\psi$ and satisfies~$\psi$.
\end{itemize}
It is easy to define a bisimulation for all the worlds reachable from $s'$ and $s_\psi$, respectively, by which each world in~$W$ is identified with itself and its replica in~$W'$, and each world in $W_\psi$ is identified with the corresponding world in $W'_\varphi$.

Suppose $\B_a\varphi$ is inconsistent with~$K$, then so is $\B_a\psi$ since $\psi\models\varphi$. It follows that $K\fm\B_a\psi = K_{(M_\emptyset,s)+\B_a\psi}$ and $(K\fm\B_a\varphi)\fm\B_a\psi=K_{((M_\emptyset,s)+\B_a\varphi)+\B_a\psi}$. Similar to the above it can be shown that the resulting pointed Kripke models entail the same set of beliefs.

Suppose $\B_a\varphi$ is consistent with~$K$ but $\B_a\psi$ is not, then $B_a\psi$ is also inconsistent with $K+\B_a\varphi$. It follows that $(K\fm \B_a\varphi)\fm\B_a\psi = K_{(M_\emptyset,s)+\B_a\psi} = K\fm \B_a\psi$.

\subsubsection{DP2 -- Irrelevance of superseded beliefs}
Generalized full-meet does \emph{not\/} satisfy weak DP, for the following reason\/: If $\B_b\psi$ is consistent with the current belief set~$K$ then $K\fm\B_b\psi=K+\B_b\psi$, hence by Lemma~\ref{lem:belief-other}, all agents retain all their beliefs when revising $K$ set by $b$'s new belief. But if $K$ is revised by $\B_a\varphi$ first and $\models\B_b\psi\rightarrow\B_a\varphi$ holds then $(K\fm\B_b\varphi)\fm \B_b\psi=K_{(M_\emptyset,s)}+\B_b\psi$, which implies that all other agents' beliefs have been erased. The fact that full-meet does not satisfy DP2 mirrors a result in classical, single-agent belief revision~\cite{jin2007iterated}.

For this reason, full-meet multi-agent belief revision satisfies only one direction of DP2, namely, revising by a belief that is then superseded by a second, contradictory belief never introduces more beliefs than revision with the second belief directly. To show this, we make a case distinction\/:
\begin{enumerate}
    \item Suppose $\models\varphi\leftrightarrow\bot$ then $\B_a\neg\varphi\in K$, hence $K\fm\B_a\varphi=K_{(M_\emptyset,s)}+\B_a\bot$, that is, there is no $a$-accessible world in~$s$ while all other agents' first-degree beliefs are tautological proposition formulas. Consequently, $(K\fm\B_a\varphi)\fm\B_b\psi\subseteq K\fm\B_b\psi$.
    \item Otherwise, $\models\B_b\psi\rightarrow\B_a\varphi$ implies $a=b$ and $\psi\models\neg\varphi$, hence $\varphi\models\neg\psi$. It follows that $\B_b\psi\not\in K\fm \B_a\varphi$, hence $(K\fm\B_a\varphi)\fm\B_b\psi=K_{(M_\emptyset,s)}+\B_b\psi$, which implies
    $(K\fm\B_a\varphi)\fm\B_b\psi\subseteq K\fm\B_b\psi$. 
\end{enumerate}

\subsubsection{DP3 -- Consistency preservation across revisions}
If $\models\varphi\leftrightarrow\top$ then DP3 holds trivially. Otherwise, $\B_a\varphi\in K\fm \B_b\psi$ implies that $\B_b\psi$ is consistent with $K*B_a\varphi$ unless $\models\psi\rightarrow\bot$. But if the latter is true, then $K\fm \B_b\psi=(K\fm\B_a\varphi)\fm \B_\psi=K_{(M_\emptyset,s)+\B_b\bot}$. In either case, DP3 follows.

\subsubsection{DP4 -- Minimal change when reaffirming a belief}
It is easy to see that DP4 holds if $\B_b\psi$ is inconsistent with $K\fm\B_a\varphi$. In case it is consistent, DP4 follows from Lemma~\ref{lem:belief-other} and~\ref{lem:first-degree-belief-agent}.

\subsubsection{Independence}
While generalized full-meet satisfies both DP3 and DP4, it does \emph{not\/} satisfy the stricter postulate of Independence\/: Consider two agents
$a\not=b$ and a satisfiable but non-tautological proposition formula~$\varphi$. Suppose further that
$\B_a\varphi$ is consistent with $K$ but $\B_b\psi$ is not, then $\B_b\psi$ is also inconsistent with $K\fm \B_a\varphi$. It follows that $\B_a\varphi\not\in (K\fm \B_a\varphi) \fm \B_b\psi$. But in this case, $a$ only believes in tautological proposition formulas in the revised set $K\fm \B_b\psi=K_{(M_\emptyset,s)}+\B_b\psi$, hence $\B_a\neg\varphi\not\in K\fm\B_b\psi$, thus violating the postulate.
 
\section*{$*_{\textnormal{ev}}$ and Generalized AGM Postulates (Theorem~\ref{th:ev-agm})}

In the following, we will assume that $K$ is consistent and represented by $(M,s)$ with  
 $M = (W, \{R_x\}_{x \in \calag}, \pi)$. 
We use $(M,s) \ev \B_a \varphi$ and 
$(M^\varphi,s^\varphi)$ interchangably  
and $M^\varphi = (W^\varphi, \{R^\varphi_x\}_{x \in \calag}, \pi^\varphi)$.
Successive revisions such as $(K \ev \B_a \varphi) \ev \B_a \psi$ 
will be denoted by $(M^{\varphi; \psi}, s^{\varphi; \psi})$ etc. 
The superscript $^\varphi$ is also attached to 
elements of $\mathbf{\Sigma}^a(\varphi)$ whenever it is needed to differentiate 
two revisions, e.g., $\mathbf{\Sigma}^a(\varphi)$ and 
$\mathbf{\Sigma}^a(\psi)$. 
Furthermore, $R_a(s|\varphi) = \{u \mid u \in R_a(s), \pi(u) \models \varphi\}$.
We first prove some properties of the $\star$ operation that will be used in the proofs of some postulates. 
In the proofs, we use conjunction of literals and set of literals. 
Furthermore, whenever we refer to conjunction of literals such as $\varphi$, we assume that $\varphi$ is consistent in the classical sense, i.e., $\varphi \not\models \bot$. 
For a set of literals $\varphi$,
$\neg \varphi$ denotes $\{\neg \ell \mid \ell \in \varphi\}$. 
Let $R$ be a set of inference rules, $\varphi$ and $\psi$ be sets of literals, $p$ and $q$ are literals, 
and $u, u'$ be consistent interpretations of $\calp$.
\begin{lemma}
\label{lem:update1}
If $R = \emptyset$ then the following holds: 
\begin{itemize}
    \item $u \star \varphi = (u \setminus \neg \varphi) \cup \varphi$; 
    \item if $u \setminus \neg \varphi \models \psi$ then 
        $u \star (\varphi \wedge \psi) = u \star \varphi \cup \psi$; and
    \item if $u \star \varphi \cup \psi$ is consistent then $u \star \varphi \cup \psi = u \star (\varphi \wedge \psi)$. 
\end{itemize}

\end{lemma}
\begin{proof}
    Let $u' = (u \setminus \neg \varphi) \cup \varphi$. \\
    Because $u' \cap u = (u \setminus \neg \varphi) \cup (\varphi \cap u)$, \\ 
    we have $(u' \cap u) \cup \varphi =  (u \setminus \neg \varphi) \cup ((\varphi \cap u) \cup \varphi) = u'$. 
    Furthermore, because $u$ is consistent, we can conclude that $u'$ is also consistent as for every $p \in \calp$, either $p$ or $\neg p$ belongs to $u'$ but not both. 

    $u \setminus \neg \varphi \models \psi$ implies 
    $(u \setminus \neg \varphi) \cap \neg \psi = \emptyset$. 
    Therefore $u \setminus (\neg \varphi \cup \neq \varphi) = (u \setminus \neg \varphi)$.
    This implies the conclusion of the property.

    The third item follows from the first two items.  
\end{proof}

\begin{lemma}
\label{lem:update2}
If $R = \{p \rightarrow q, q \rightarrow p\}$ then 
\begin{itemize}
    \item if $\{p,q\} \cap \varphi = \emptyset$ then $u \star \varphi = (u \setminus \neg \varphi) \cup \varphi$; and 
    \item if $\{p,q\} \cap \varphi \ne \emptyset$ then $u \star \varphi = (u \setminus (\neg \varphi \cup \{\neg p, \neg q\})) \cup \varphi \cup \{p,q\}$. 
\end{itemize}
\end{lemma}
\begin{proof}
Let $u' = u \star \varphi$.
\begin{itemize}
    \item If $\{p,q\} \cap \varphi = \emptyset$ then, as shown in Lemma~\ref{lem:update1}, 
    $u'$ is an interpretation for $\calp$.
    Assume $p \in u'$. It implies $p \in u \setminus \neg \varphi$, and thus $p, q \in u$ and $q \in u'$.
    Similarly, $q \in u'$ implies $p \in u'$.    
    This shows that $u' = C_R(u')$. 
    \item If $\{p,q\} \cap \varphi \ne \emptyset$ then it is clear that $u' = C_R(u')$ because $\{p,q\} \subseteq u'$. 
    Similar arguments to the proof of Lemma~\ref{lem:update1} lead to 
    $u' = C_R((u \cap u') \cup \varphi)$.
\end{itemize}
This completes the proof.
\end{proof}

The following lemma is similar to Lemma~\ref{lem:update2}. 
    
\begin{lemma}
\label{lem:update3}
\begin{enumerate}
    \item 
If $R = \{p \rightarrow q, \neg q \rightarrow \neg p\}$ then 
\begin{itemize}
    \item if $\{p, \neg q\} \cap \varphi = \emptyset$ then $u \star \varphi = (u \setminus \neg \varphi) \cup \varphi$; 
    \item if $p \in \varphi$ then $u \star \varphi = (u \setminus \neg \varphi \cup \{\neg q\}) \cup \varphi \cup \{q\}$; and 
    \item if $\neg q \in \varphi$ then $u \star \varphi = (u \setminus \neg \varphi \cup \{p\}) \cup \varphi \cup  \{\neg p\}$. 
\end{itemize}
\item If $R = \{p \rightarrow \neg q, q \rightarrow \neg p\}$ then 
\begin{itemize}
    \item if $\{p, q\} \cap \varphi = \emptyset$ then $u \star \varphi = (u \setminus \neg \varphi) \cup \varphi$; 
    \item if $p \in \varphi$ then $u \star \varphi = (u \setminus \neg \varphi \cup \{q\}) \cup \varphi \cup \{\neg q\}$; and 
    \item if $q \in \varphi$ then $u \star \varphi = (u \setminus \neg \varphi \cup \{p\}) \cup \varphi \cup  \{\neg p\}$. 
\end{itemize}
\end{enumerate}
\end{lemma}

\subsubsection{Closure} By our definition, it is clear that  
\[ K \ev{}  \B_a\varphi = \Cn(K \ev{}  \B_a\varphi).\] 

\subsubsection{Success}
We show that $K\ev{} \B_a\varphi \models \B_a \varphi$. 
Consider two cases: 
\begin{enumerate}
    \item $(M,s) \models \neg \B_a \neg \varphi$, i.e., 
    $a$ does not believe in $\neg \varphi$ before the revision.
    This implies that $s^\varphi = (s, \sigma)$. 
    Furthermore, because $K$ is consistent, we have that $R_a(s) \ne \emptyset$
    and there exists some $u \in W$ such that $(s,u) \in R_a$ and
    $\pi(u) \models \varphi$ which implies that $((s,\sigma), (u,\sigma_a)) \in R^\varphi_a$, 
    i.e., $R^\varphi(s^\varphi) \ne \emptyset$. 
    In addition, if $((s, \sigma), (u, \rho)) \in R^\varphi_a$ then $\rho = \sigma_a$, and hence, $\pi(u) \models \varphi$, because 
    $E_a$ contains only one element related to $\sigma$, $(\sigma, \sigma_a)$, and $(s,u) \in R_a$.
    Thus, we have that $(M^\varphi, s^\varphi) \models \B_a \varphi$. 

    \item $(M,s) \models \B_a \neg \varphi$.
    This implies that $s^\varphi = (s, \delta)$. 
    Again,  because $K$ is consistent, we have that $R_a(s) \ne \emptyset$
    and for each $u \in R_a(s)$, $\pi(u) \models \neg \varphi$.  
    
    By the construction of $(M^\varphi, s^\varphi)$, 
    $R^\varphi_a(s^\varphi) \ne \emptyset$. 
    
    Consider $u' \in R^\varphi_a(s^\varphi)$. We have that 
    $u ' = (u, \delta_a)$ for some $u \in R_a(s)$, and hence, $\pi^\varphi(u') = C_R((\pi(u) \cap \pi^\varphi(u')) \cup \{\varphi\})$, which implies $\varphi \in \pi^\varphi(u')$.  
    Therefore, we have that $(M^\varphi, s^\varphi) \models \B_a \varphi$.    
\end{enumerate}

\subsubsection{Inclusion}  
Due to Lemma~\ref{lem:belief-other} and Lemma~\ref{lem:revision-other}, it suffices to show that if 
$\B_a \phi \in K \ev{}  \B_a\varphi$ then $\B_a \phi  \in K + \B_a\varphi$ holds.  
By definition of $(M',s')=K + \B_a\varphi$, we have that $R'_a(s') = \emptyset$, i.e., $a$ has inconsistent belief in $K + \B_a\varphi$    
if $(M,s) \models \B_a \neg \varphi$. Thus, the 
postulate holds trivially in this case. 
Therefore, we only need to consider the case 
that $(M,s) \models  \neg \B_a \neg \varphi$. In this case, as shown in Case 1 above, we have that 
$s^\varphi  = (s, \sigma)$ and
$((s, \sigma), (u,\sigma_a)) \in R^\varphi_a$ implies that $(s,u) \in R_a$ and $\pi(u) \models \varphi$.
This implies that $(s^r,u^r) \in R'_a$ (Definition of Expansion). 
On the other hand, if $(s',u^r) \in R'_a$ then we can also conclude that 
$((s, \sigma), (u,\sigma_a)) \in R^\varphi_a$. 
This implies that $(M',s') \models \B_a \phi$ iff $(M^\varphi ,s^\varphi) \models \B_a \phi$.  

\subsubsection{Vacuity} The proof of this postulate is the second part of the proof for \textbf{Inclusion}.  

\subsubsection{Consistency} 
Since $K$ and $\B_a\varphi$ are consistent, 
$R_a(s) \ne \emptyset$ and exactly one  out of two cases in the proof for \textbf{Success} occurs. 
In any case, $(M^\varphi,s^\varphi)$ is defined and $R^\varphi_a(s^\varphi) \ne \emptyset$. Together with Lemma~\ref{lem:revision-other},
we have that $K \ev{}  \B_a \varphi$ is consistent. 

\subsubsection{Extensionality} Assume that $ \models \B_a\varphi \leftrightarrow \B_b\psi$. 
This implies that $\varphi \leftrightarrow \psi$ and $a = b$. 
It is easy to see that 
$(e, \eta^\varphi) \in W^\varphi$ iff $(e, \eta^\psi) \in W^\psi$.
Similarly, $((u, \eta^\varphi), (v, \xi^\varphi)) \in R^\varphi_x$ iff 
$((u, \eta^\psi), (v,\xi^\psi)) \in R^\psi_x$ for any agent $x \in \calag$.
This shows that there is a bijection between $(M^\varphi, s^\varphi)$ and $(M^\psi, s^\psi)$ which proves that $\ev{} $ satisfies this postulate. 

\subsubsection{Superexpansion}
Consider a proposition formula $\phi$,  Lemma~\ref{lem:revision-other} and 
Lemma~\ref{lem:belief-other} 
shows that $ K \ev{}  \B_a(\varphi \wedge q) \models \B_x \phi$ 
iff $(K \ev{}  \B_a\varphi) + \B_a \psi  \models \B_x \phi$ for $x \ne a$. 
So, to prove that $\ev$ satisfies this postulate, it suffices to show that  
$ K \ev{}  \B_a(\varphi \wedge \psi) \models \B_a \phi$ implies 
that $(K \ev{}  \B_a\varphi) + \B_a \psi  \models \B_a \phi$.

Let $(M',s')$ denotes $(M^{\varphi},s^{\varphi}) + \B_a \psi$. 
The proof considers two cases, similar to the proof of \textbf{Success}. 

\begin{enumerate}
    \item $(M,s) \models \neg \B_a \neg (\varphi \wedge \psi)$.
    In this case, $s^{\varphi \wedge \psi} = (s, \sigma^{\varphi \wedge \psi})$ and $(s^{\varphi \wedge \psi},u^{\varphi \wedge \psi}) \in R^{\varphi \wedge \psi}_a$ 
    where $u^{\varphi \wedge \psi} = (u,\sigma^{\varphi \wedge \psi}_a)$ iff $(s,u) \in R_a$ and $\pi(u) \models (\varphi \wedge \psi)$.
    Since $(M,s) \models \neg \B_a \neg (\varphi \wedge \psi)$, we have 
    that $(M,s) \models \neg \B_a \neg \varphi$, i.e., $u \in R_a(s|\varphi \wedge \psi)$. 
    This implies that $s^{\varphi} = (s, \sigma^{\varphi})$, 
        and $(s^{\varphi},u^{\varphi}) \in R^{\varphi}_a$ where $u^{\varphi} = (u,\sigma^{\varphi}_a)$ iff $(s,u) \in R_a$ and $\pi(u) \models \varphi$, i.e., $u \in R_a(s|\varphi)$.
        Because $R_a(s|\varphi) = R_a(s|\varphi \wedge \psi) \cup R_a(s|\varphi \wedge \neg \psi)$ and 
        the construction of $(K \ev{}  \B_a\varphi)   \models \B_a \phi$ we have 
        $R'_a(s') = \{u^r \mid u \in R_a(s), \pi(u) \models \varphi \wedge \psi\}$ which proves
        the consequence of the postulate in this case. 
        
    \item $(M,s) \models \B_a \neg (\varphi \wedge \psi)$. 
    In this case, $s^{\varphi \wedge \psi} = (s, \delta^{\varphi \wedge \psi})$ and 
    for every $u^{\varphi \wedge \psi}$ such that     
    $(s^{\varphi \wedge \psi},u^{\varphi \wedge \psi}) \in R^{\varphi \wedge \psi}_a$ holds, we have 
    $u^{\varphi \wedge \psi} = (u,\delta^{\varphi \wedge \psi}_a)$ iff 
    $(s,u) \in R_a$.
    Therefore, we can conclude that $K \ev{}  \B_a(\varphi \wedge \psi) \models \B_a \phi$ iff  for every $u 
    \in R_a(s)$, $\pi(u) \star (\varphi \wedge \psi) \models \phi$.
    There are two cases: 
    \begin{enumerate}
        \item $(M,s) \models \B_a \neg \varphi$. 
        In this case, similar arguments to the above allow us to conclude that 
        $(s^{\varphi},u^{\varphi}) \in R^{\varphi}_a$ iff 
        $u^{\varphi} = (u,\delta^{\varphi}_a)$ for some $(s,u) \in R_a$ and $\pi(u^\varphi) = \pi(u)\star \varphi$.
        Because of the construction of $(K \ev{}  \B_a\varphi) + \B_a\psi$, we have two cases: 
        \begin{enumerate}
        \item $K \ev{}  \B_a\varphi \models \B_a \neg \psi$. In this case, 
        $a$ has inconsistent beliefs in $(K \ev{}  \B_a\varphi) + \B_a\psi$. Thus, 
        the postulate holds trivially. 
        \item $K \ev{}  \B_a\varphi \models \neg \B_a \neg \psi$.  
        In this case, we can conclude that   $R'_a(s') \ne \emptyset$. 
        Furthermore, $(s',u') \in R'_a$ iff $u' = u^\varphi$ and 
        $\pi(u^\varphi) \models \psi$. 
        Lemma~\ref{lem:update1}, Item 2, allows us to conclude that  
        $K \ev{}  \B_a(\varphi \wedge \psi) \models \B_a \phi$ 
        implies 
        $K \ev{}  \B_a\varphi + \B_a \psi \models \B_a \phi$.
        \end{enumerate}
                
        \item $(M,s) \models \neg \B_a \neg \varphi$. Similar arguments as in Case 1 of the proof of \textbf{Success}, we have $s^{\varphi} = (s, \sigma^{\varphi})$, and $(s^{\varphi},u^{\varphi}) \in R^{\varphi}_a$ where $u^{\varphi} = (u,\sigma_a^{\varphi})$ 
        iff $(s,u) \in R_a$ and $\pi(u) \models \varphi$ which implies that   $\pi^{\varphi}(u^{\varphi}) \models \neg \psi$ 
        for every $(s^{\varphi},u^{\varphi}) \in R^{\varphi}_a$. This means that $a$'s belief is inconsistent in  
        $(K \ev{}  \B_a\varphi) + \B_a\psi$, and thus, the postulate also holds in this case. 
    \end{enumerate}
\end{enumerate}

\subsubsection{Subexpansion} 
Similar to the proof  of \textbf{Superexpansion}, it suffices to show 
that 
    $(K \ev{}  \B_a\varphi) + \B_a \psi  \models \B_a \phi$ iff 
    $(K \ev{}  \B_a (\varphi \wedge  \psi)  \models \B_a \phi$.
We consider two cases: 
\begin{enumerate}
    \item $(M,s) \models \neg \B_a \neg \varphi$.  
    Because of $\B_a \neg  \psi \not\in K \ev{} \B_a \varphi$, 
    we can conclude that there exists some  
    $u \in R_a(s)$ such that $\pi(u) \models \varphi \wedge  \psi$.
    The proof of the postulate for this case is then similar to the proof 
    in of \textbf{Superexpansion}, Case 1. 

    \item $(M,s) \models \B_a \neg \varphi$. This also implies that 
    $(M,s) \models \B_a \neg (\varphi \wedge  \psi)$.
    This means that $s^\varphi = (s, \delta)$ 
    and for every $u^\varphi \in W^\varphi$ such that 
    $(s^\varphi, u^\varphi) \in R^\varphi_a$, it holds that 
    $u^\varphi = (u, \delta_a^\varphi)$ for some $u \in R_a(s)$ 
    and $\pi^\varphi(u^\varphi) = \pi(u) \star \varphi$.

    Because $(M,s) \ev \B_a \varphi \not\models \B_a \neg  \psi$, 
    we conclude that there exists some $u \in R_a(s)$ such that 
    $\pi(u) \star \varphi \models  \psi$. It means that $R'_a$, 
    the accessibility relation of $a$ in $(M,s) \ev \B_a \varphi + \B_a  \psi$, 
    consists of $(s', (u^\varphi)^r) \in R'_a$ such that 
    $u^\varphi = (u, \delta_a^\varphi)$,   
    $\pi^\varphi(u^\varphi) = \pi(u) \star \varphi$, and $\pi^\varphi(u^\varphi) \models  \psi$.

    Lemma~\ref{lem:update1} shows that   
    $\pi(u) \star \varphi +  \psi = \pi(u) \star (\varphi \wedge  \psi)$ 
    if $\pi(u) \star \varphi \models  \psi$.
    This implies that $(s^{\varphi \wedge  \psi}, u^{\varphi \wedge  \psi}) \in 
    R^{\varphi \wedge  \psi}_a$. This, together with the arguments similar to that 
    used in Casse 2(a.ii) of the previous proof, allows us to conclude that   
    $(K \ev{}  \B_a\varphi) + \B_a \psi  \models \B_a \phi$ iff 
    $(K \ev{}  \B_a (\varphi \wedge  \psi)  \models \B_a \phi$.
\end{enumerate}

\section*{$*_{\textnormal{ev}}$  and Generalized DP Postulates (Theorem~\ref{th:ev-dp})}

For these postulates, we will restrict the formulas $\varphi$ and $\psi$ to single literals. 

\subsubsection{DP1 -- Successive revision respect} 
Since $ \models \B_b  p \rightarrow \B_a  q$, we have that 
$ p \rightarrow  q$ and $a$ and $b$ are identical. 
Therefore, we will show that $\ev$ satisfies this postulate by showing that 
if $\models \B_a p \rightarrow \B_a q$ and $K \ev{} B_a q \not\models \B_a \neg p$ hold then  
$(K \ev{} \B_a q) \ev{}  \B_a p \doteq K\ev{} \B_a p$. 

\begin{enumerate}
    \item $(M,s) \models \neg \B_a \neg  q$. 
    So, if $(s^ q,u^ q) \in R^ q_a$ then  
    $u \in R_a(s)$ and $\pi(u) \models   q$ and 
    $\pi^q(u^q) = \pi(u)$. 

    Since $(M,s) \ev \B_a  q \models \neg \B_a \neg  p $, for $u^{ q; p} \in W^{ q; p} $ such that 
        $(s^{ q; p}, u^{ q; p}) \in R^{ q; p}_a$     
        iff $(s^{ q}, u^{ q}) \in R^{ q}_a$ and    
        $\pi^{ q}(u^{ q}) \models  p$.
        This implies that $u^{ p} \in W^ p$ and $(s^ p, u^ p) \in R^ p_a$. 
        
        $(M,s) \ev \B_a  q \models \neg \B_a \neg  p $ also 
        implies that $(M,s) \models \neg \B_a \neg  p$. 
        Therefore, for every $u\in R_a(s)$ such that $u^{ p} \in W^ p$ and $(s^ p, u^ p) \in R^ p_a$,
        we can conclude that  $(s^{ q; p}, u^{ q; p}) \in R^{ q; p}_a$ because $\models  p \rightarrow  q$. 
         
        The above imply that  
$(K \ev{} \B_a q) \ev{}  \B_a p \doteq  K\ev{} \B_a p$ for this case. 
        
    \item $(M,s) \models \B_a \neg  q$. 
    This implies that $(M,s) \models \B_a \neg  p $ and $s^ q = (s, \delta^ q)$ 
    and $(s^ q,u^ q) \in R^ q_a$ iff $u^ q = (u,\delta^ q_a)$ and 
    $\pi^ q(u^ q) = \pi(u) \star  q$.
    Because $(M,s) \models \B_a \neg  p$, we have that $\pi^ q(u^ q) \models \neg  p$ (Lemma ~\ref{lem:update3}). 
    This implies that $s^{ q; p} = ((s, \sigma^ q), \delta^{ q; p})$ and 
    for every $u \in R_s(a)$,
     $(s^{ q; p}, u^{ q; p}) \in R^{ q; p}_a$, 
     $(s^ q, u^ q) \in R^ q_a$, and 
    $\pi^{ q; p}(u^{ q; p}) = (\pi(u) \star  q) \star  p$.
    On the other hand, by construction of $(M,s) \ev{} p$,   
    $s^ p = (s, \delta^ p)$ 
    and, for every $u \in W$, 
    $(s^ p,u^ p) \in R^ p_a$ where $u^ p = (u,\delta^ p_a)$ and 
    $\pi^ p(u^ p) = \pi(u) \star  p$.
    Lemma~\ref{lem:update3} indicates that
    $\pi(u) \star  p = (\pi(u) \star  q) \star p$ for every $u \in R_a(s)$.
    This proves that  
$(K \ev{} \B_a q) \ev{}  \B_a p \doteq  K\ev{} \B_a p$ for this case. 
        
\end{enumerate}

\subsubsection{DP2 -- Irrelevance of superseded beliefs}
As in \textbf{DP1},  
$\models\B_b p \rightarrow \B_a\neg q$ implies that $a = b$ and $ p \rightarrow \neg  q$. 
As such,   we need to show that  
if $K \models \B_a p \vee \B_a q$ then 
$\B_a p \models \B_a \neg q\ \Rightarrow\ (K\ev{} \B_a q) \ev{}  \B_a p \doteq K\ev{} \B_a p$. 

\begin{enumerate}
    \item $(M,s) \models \B_a p$. Then, $(M,s) \models \B_a \neg q$.
    Therefore, $s^ q = (s, \delta^ q)$ and 
    for every $u \in R_a(s)$, 
    $(s^q,u^q) \in R^q_a$ where $u^q=(u,\delta_a^q)$ 
    and $\pi^q(u^q) = \pi(u) \star q$.
    Lemma~\ref{lem:update3} implies that $\pi^q(u^q) \models \neg p$. 
    Thus, we can conclude that 
    for every $u \in R_a(s)$, 
    $(s^{q;p},u^{q;p}) \in R^{q;p}_a$ where $u^{q;p}=(u^q,\delta_a^{q;p})$ 
    and $\pi^{q;p}(u^{q;p}) = \pi^q(u^q) \star p = (\pi(u) \star q) \star p$.
    It is easy to see that  $(\pi(u) \star q) \star p = u$ which, 
    together with the fact $(M,s) \models \B_a p$, proves that 
    $(K\ev{} \B_a q) \ev{}  \B_a p \doteq K\ev{} \B_a p$. 
    
    \item $(M,s) \models \B_a q$.  Then, $(M,s) \models \B_a \neg p$.
    Then, we can easily show that for every $u \in R_a(s)$, 
    \begin{itemize}
        \item $u^{q;p} = ((u, \sigma^q), \delta^{q;p}_a) \in R_a(s^{q;p}$ 
        and $\pi^{q;p}(u^{q;p}) = \pi^q(u^q) \star p = (\pi(u) \star q) \star p = \pi(u) \star p$; 
        and 
        \item $u^{p} = (u, \delta^{p}_a) \in R_a(s^{p}$ and 
        $\pi^{p}(u^{p}) = \pi(u) \star p$.  
    \end{itemize}
    The above two properties show that  
    $(K\ev{} \B_a q) \ev{}  \B_a p \doteq K\ev{} \B_a p$. 
\end{enumerate}

\subsubsection{DP3 -- Consistency preservation across revisions}
The operator $\ev{} $ satisfies this postulate due to Lemma~\ref{lem:revision-other} 
and the \textbf{Success} postulate: $\B_a\varphi \in (K\ev{} \B_a\varphi)$. 

\subsubsection{DP4 -- Minimal change when reaffirming a belief}
Assume that $a$ has consistent belief in $K$. 
Because of Lemma~\ref{lem:revision-other}, 
$\B_a\varphi \in (K\ev{} \B_a\varphi) \ev{}  \B_b \psi$. 
if $\B_a \neg \varphi \in (K\ev{}\B_a\varphi) \ev{}\B_b \psi$, 
it means that $a$ has inconsistent belief in $K\ev{}\B_a\varphi$
which implies that $a$ has inconsistent belief in $K$. 
This contradicts the assumption. 
This proves the postulate. 

\subsubsection{Independence}
Similar to \textbf{DP4}.

\section*{Properties of  $*_{\textnormal{rb}}$ (Theorem~\ref{th:fb-all})}

To prove this theorem, we observe the following properties of $(M,s) \fb \varphi$: 
for every $u \in W$ and $\lambda \in \Sigma$, $u^\varphi = (u, \lambda) \in W^\varphi$. 
Furthermore, $(s^\varphi, u^\varphi) \in R^\varphi_a$ iff $(s,u) \in R_a(s)$ 
and $\pi^\varphi(u^\varphi) = \pi(u) \star \varphi$. It is easy to see that 
the generalized AGM postulates hold because of the following properties of the propositional belief operator $\star$
\begin{itemize}
    \item  for every $u$, $\pi(u^\varphi) = \pi(u) \star \varphi \models \varphi$;  
    \item  for every $u$ if $\pi(u) \star \varphi \not\models \neg \psi$ then 
    $\pi(u) \star  \varphi \wedge \psi  = (\pi(u) \star \varphi) + \psi$; 
    and if $\pi(u) \star \varphi \models \neg \psi$ then $(\pi(u) \star \varphi) + \psi$ is inconsistent. 
\end{itemize}
The generalized DP postulates also hold because of the properties of $\star$ (Lemma~\ref{lem:update2}-\ref{lem:update3}) and the fact that successive revisions do not remove from or add elements to the accessibility relation of $a$.  

\fi

\end{document}